\documentclass{article}
\usepackage{iclr2027_conference,times}
\usepackage{amsmath,amsfonts,bm}

\def\eqref#1{equation~\ref{#1}}
\def\1{\bm{1}}

\DeclareMathAlphabet{\mathsfit}{\encodingdefault}{\sfdefault}{m}{sl}
\SetMathAlphabet{\mathsfit}{bold}{\encodingdefault}{\sfdefault}{bx}{n}

\newcommand{\R}{\mathbb{R}}

\usepackage{hyperref,url,booktabs,multirow,array,amsmath,amssymb,mathtools,microtype,xcolor,placeins}
\usepackage[ruled,vlined]{algorithm2e}
\usepackage{algpseudocode}
\usepackage{tikz}
\usepackage{iftex}
\ifXeTeX
  \usepackage{fontspec,xeCJK}
  \IfFontExistsTF{SimSun}{\setCJKmainfont{SimSun}[AutoFakeBold=true]}{\setCJKmainfont{FandolSong-Regular}[BoldFont=FandolSong-Bold]}
  \IfFontExistsTF{SimHei}{\setCJKsansfont{SimHei}}{\setCJKsansfont{FandolHei-Regular}}
  \IfFontExistsTF{FangSong}{\setCJKmonofont{FangSong}}{\setCJKmonofont{FandolFang-Regular}}
  \newenvironment{CJK}[2]{}{}
\else
  \usepackage{CJKutf8}
\fi
\usepackage{graphicx,longtable,amsthm}
\newtheorem{theorem}{Theorem}
\newtheorem{proposition}{Proposition}
\usetikzlibrary{arrows.meta,positioning,fit,matrix}
\hypersetup{colorlinks=true,citecolor=blue,linkcolor=blue,urlcolor=blue,
 pdftitle={BriLLM-MetaPred: A Representation-Prediction Unity Language Model with Adaptive Self-Measurement},
 pdfauthor={Junyuan Liu, Hongyue Lin, Jiale Hong, Yi Zhao, Hai Zhao}}

\SetKwInOut{kwIn}{Input}
\SetKwInOut{kwOut}{Output}

\newcommand{\Fast}{\boldsymbol{\phi}}
\newcommand{\Slow}{\boldsymbol{\theta}}

\newcounter{algorithm}

  \title{BriLLM-MetaPred: \\ A Representation-Prediction Unity\\Language Model with\\Adaptive Self-Measurement}

\author{Junyuan Liu$^{\#}$ \quad Hongyue Lin$^{\#}$ \quad Jiale Hong$^{\#}$\\
\textbf{Yi Zhao \quad Hai Zhao}\thanks{Supported by the Shanghai Jiao Tong University 2030 Initiative.}\\
Institute of Artificial General Intelligence\\
Shanghai Jiao Tong University, Shanghai, China\\
zhaohai@cs.sjtu.edu.cn}
\iclrfinalcopy
    
\begin{document}
\begin{CJK}{UTF8}{gbsn}
\raggedbottom
\maketitle
\fancyhead[L]{Preprint}

\begin{abstract}
We present \textbf{BriLLM-MetaPred}, the first engineering implementation of the \emph{MetaPred} machine-learning framework established by a mathematical theory of AGI, instantiated for language modeling. The theory argues that MetaPred is both necessary and sufficient as a machine-learning primitive for general cognition.
Taking this theoretical conclusion as a design premise, we report an implementation of the primitive for language modeling. A single continuously evolving dynamical state serves as both the context representation and the state from which predictions are read out. Continuous measurement of its own state reduces finite-depth inconsistency, maintains internal consistency, and provides an endogenous dynamical self-supervision signal in fixed-point form.
The single dynamical state of MetaPred intrinsically couples training and inference. As a machine-learning paradigm, MetaPred therefore exhibits symmetry across almost all key elements of machine learning, an intuitive and observable property of the mathematical optimum over the space of all machine-learning paradigms. BriLLM-MetaPred is neither another recurrent language model nor an extension of BriLLM. It provides a concrete bridge between two prior contributions: the construction and motivation of the MetaPred primitive provided by AGI theory, and the SiFu learning paradigm provided by BriLLM \citep{zhao2025brillm}. The latter's node-signal-flow design enables the primitive to run on currently available software and hardware accelerators. Our implementation of BriLLM-MetaPred achieves predictive performance comparable to a GPT baseline of the same size. This provides initial evidence for the engineering feasibility of the MetaPred framework established by the mathematical theory of AGI. MetaPred is the first machine-learning paradigm derived entirely from mathematical theory rather than assembled and incrementally improved through human design---and perhaps the last such paradigm.
\end{abstract}

\section{Introduction}
\label{sec:intro}

The mathematical theory of AGI \citep{zhao2026agitheory} establishes a single machine-learning primitive, MetaPred, as a necessary and sufficient condition for general cognition. MetaPred is not characterized by a particular network topology. Its defining property is that the representations formed by the model and the predictions it produces are two aspects of the same persistent dynamical state. Under endogenous fixed-point supervision, this state must maintain internal consistency by reducing its finite-depth inconsistency through self-measurement. We report the first engineering implementation of MetaPred for language modeling. We call this implementation BriLLM-MetaPred because it is built on the SiFu learning paradigm introduced by BriLLM \citep{zhao2025brillm}. SiFu provides the node-signal-flow carrier on which the MetaPred primitive can execute; this paper provides the primitive-level objective and a finite-budget self-measurement mechanism that make the carrier satisfy the requirements of MetaPred.

BriLLM \citep{zhao2025brillm} provides a node-signal-flow implementation that allows the MetaPred primitive to run on current accelerators, while the theory provides the primitive itself. This paper provides an executable bridge between them. The model is therefore the first language-model implementation of MetaPred, rather than a competitor to Transformer-based language models in the usual architectural sense or an incremental improvement to BriLLM.

Specifically, the model is a node-state language model with a single recurrent tensor. Token projection vectors are written into fixed interface rows, a recurrent operator updates the entire tensor through dynamical evolution, and the next-token distribution is read from designated output rows and a bounded auxiliary pool. The same evolving tensor carries both the input-conditioned representation and the state used for readout. Readout uses the output node region and several auxiliary sources, and maps them to a vocabulary distribution through a separate output projection. This projection is a learnable fully connected layer, but its input comes from the same dynamical state. Thus, representation and the information used for prediction come from a single dynamical state, and the output projection serves only as a parameterized mapping.

MetaPred has an endogenous fixed-point prediction objective that continuously maintains the consistency of its own state. External labels constrain only the output distribution and do not determine the trajectory of the internal state. The endogenous consistency loss therefore does not assume that the learned dynamics converge, nor that a smaller internal residual implies a correct prediction. As required by the mathematical theory, MetaPred must provide an adaptive predictive self-measurement mechanism that allows greater flexibility in improving efficiency and performance during both training and inference. Starting from the same input conditions, this mechanism compares the successor state of one evolution step with a deeper simulated state. Their difference is used to train the shared dynamics, supervise the direct-measurement process, and provide a bounded direction for state correction.

This paper makes three contributions. First, we identify and explain the MetaPred primitive as the central object of study, and instantiate the BriLLM node-signal-flow design \citep{zhao2025brillm} as its first concrete language-model implementation. Second, we define and implement self-measurement, a core mechanism of MetaPred. It constructs simulations of different lengths from the same state and uses the difference between their endpoints as an endogenous supervision signal. This constrains the consistency of the dynamics without external labels. The signal is used to train the shared dynamics, supervise direct-measurement parameters, and drive state correction. The implementation has three key features: training and inference share the same measurement process and evolution path, with no separate computational mechanisms for training and inference; measurement, updates, and correction are subject to finite budgets and dynamically scheduled by the controller according to state and history; and the gradient paths of the state, fast parameters, control parameters, and slow parameters are completely isolated. Third, we define evaluation benchmarks at approximately 13 million and 129 million parameters, with comparisons matched for task, mechanism, and training budget. At one tenth of the scale of GPT-1, our BriLLM-MetaPred implementation achieves predictive performance comparable to a GPT baseline of the same size.

\section{MetaPred as a Machine-Learning Paradigm for Implementing AGI}
\label{sec:paradigm}

According to the mathematical theory of AGI \citep{zhao2026agitheory}, MetaPred is the optimal solution obtained by optimization over all possible machine-learning paradigms under the constraints of its cognitive functional. It is also the unique solution up to isomorphism and the only paradigm type for all machine-learning models in an AGI system. It is therefore necessary to clarify the paradigm to which MetaPred belongs and distinguish it from existing representation-learning architectures and variants of the SiFu paradigm.

\paragraph{The Two Components of a Machine-Learning Paradigm}

Following the framework of \citet{zhao2025brillm}, any machine-learning paradigm can be defined by how it represents semantic objects and how it performs prediction. Mainstream representation-learning paradigms encode semantics as vectors and perform prediction by transforming these vectors with learned functions. The SiFu learning paradigm of BriLLM \citep{zhao2025brillm} instead maps semantic units to dedicated nodes and allows prediction to emerge through signal propagation among nodes. The BriLLM implementation is a spatial, component-based scheme for semantic representation, in which different components directly carry different semantics as specified by the user. This contrasts with the temporal scheme of conventional deep learning, in which the same vector changes over time.

SiFu and BriLLM jointly remove an implicit constraint of existing machine-learning paradigms: the tight binding between the model carrier and the dynamical state. Existing models generally require the architecture and dynamical state to align strictly in shape and topology and maintain this alignment throughout all evolution steps. SiFu relaxes this constraint at two levels. First, the model carrier provided by SiFu is a fully connected, bidirectional graph with maximal generality. Second, the BriLLM implementation allows the dynamical state to be only partially aligned with the model architecture: training and inference for each learning sample correspond to a particular subset of the maximal dynamical state space. Relaxing this alignment constraint allows more flexible tradeoffs between time and space costs. This is the fundamental mechanism through which BriLLM supports arbitrarily long inputs and outputs at constant model size.

\paragraph{The MetaPred Primitive}

The mathematical theory of AGI \citep{zhao2026agitheory} asks what internal relationship a system must maintain to achieve unbounded improvements in predictive performance and efficiency without unbounded external supervision. The theory argues that the answer is MetaPred.
We take this conclusion as a design premise and test whether the primitive can be executed and measured on a language task, and whether its implementation can produce reasonable results comparable to those of existing models.

The space of candidate paradigms considered by the mathematical theory of AGI is precisely the general case of weak coupling between the model carrier and the dynamical state defined by SiFu learning, where the dynamics can be defined as a fiber bundle over a topological carrier. Because MetaPred is more general than existing machine learning at the paradigm level, its engineering implementation must be based on the generalized topological structure suggested by SiFu.

MetaPred is the optimal solution for general cognition derived by the mathematical theory of AGI. It uses a single continuously evolving dynamical state to carry both representation and prediction. This is its defining distinction from all existing machine-learning paradigms. Consequently, MetaPred exhibits symmetry across all core elements of machine learning: between representation and prediction, between the training objective and endogenous supervision, and between training and inference. These symmetries are intuitive, observable properties of the mathematical optimum in the space of machine-learning paradigms.

\begin{itemize}
    \item \textbf{Representation-prediction unity}: both the semantic representation of the input context and the prediction to be read out come from the same dynamical state $\mathbf{S}_{t,u}$, without a separate encoder-decoder structure. The output projection $W_{\text{out}}$ is only a parameterized readout mapping and does not alter the unified nature of the state itself.
    \item \textbf{Endogenous self-supervision signal}: through self-measurement, the model obtains a fixed-point self-supervised loss from the difference between finite-step simulations starting from the same state. This constrains the state's internal consistency without external labels. The mechanism instantiates Lambek's fixed-point lemma in machine learning \citep{Lambek1968}.
    \item \textbf{Training-inference coupling}: the single dynamical state intrinsically couples training and inference. During training, self-measurement provides gradients for parameter updates; during inference, it provides a direction for state correction. Both share the same evolution logic, without separate training-forward and inference-forward paths.
\end{itemize}

Table~\ref{tab:paradigms} compares three paradigms in terms of their two defining dimensions and the resulting relationship between representation and prediction.

\begin{table*}[htp]
\caption{Comparison of machine-learning paradigms.}
\label{tab:paradigms}
\centering\small
\setlength{\tabcolsep}{5pt}
\begin{tabular}{>{\raggedright\arraybackslash}p{0.18\textwidth}>{\raggedright\arraybackslash}p{0.20\textwidth}>{\raggedright\arraybackslash}p{0.22\textwidth}>{\raggedright\arraybackslash}p{0.30\textwidth}}
\hline
& \textbf{Existing machine learning} & \textbf{BriLLM} & \textbf{BriLLM-MetaPred} \\
\hline
Carrier-dynamics relationship & Model architecture and dynamical state are strictly aligned & \multicolumn{2}{>{\raggedright\arraybackslash}p{\dimexpr0.52\textwidth+2\tabcolsep\relax}}{The model carrier (a fully connected graph topology) and dynamical state need not align in shape; the dynamics can evolve flexibly within the maximal space allowed by the carrier} \\
Structural carrier & Dedicated architecture (e.g., Transformer) & \multicolumn{2}{>{\raggedright\arraybackslash}p{\dimexpr0.52\textwidth+2\tabcolsep\relax}}{General fully connected graph topology with bidirectional connections between nodes} \\
\hline
Semantic representation & Fixed or learnable vectors carried by separate components & Nodes directly map to semantics, with encoding in the spatial dimension & Encoded entirely by the dynamical state $\mathbf{S}_{t,u}$, without a separate representation component \\
\hline
Prediction mechanism & Vector transformations passed from layer to layer & Dynamic emergence through signal propagation among nodes & Direct readout from the same dynamical state \\
\hline
Representation-prediction relationship & Performed independently in separate stages & Representation is carried by nodes; prediction is determined by dynamics & Fully unified: the same state performs both representation and prediction \\
\hline
Self-measurement & None & None & Endogenous finite-step self-comparison providing fixed-point self-supervision \\
\hline
Brain analogy & Loose similarities in local aspects & Similarity in semantic mapping by macroscopic components & Mechanistic consistency: representation and prediction are both realized through dynamical evolution, with coupled training and inference, matching the brain's predictive-coding mechanism \citep{friston2010free} \\
\hline
\end{tabular}
\end{table*}

Given these paradigm-level differences, MetaPred differs from any existing machine-learning paradigm or model, even when their descriptions share terms such as recurrence, self-measurement, attention, dynamics, representation, or prediction.

In our concrete implementation, BriLLM-MetaPred is a generative language model implementing the MetaPred algorithmic framework on the SiFu learning carrier. At the paradigm level, it is fundamentally different from GPT \citep{radford2018improving}, which is based on Transformer \citep{vaswani2017attention}. Its backbone structure follows from optimization in mathematical theory and allows greater generality in both the model carrier and the realization of the dynamics than the existing GPT paradigm.

\paragraph{Paradigm-Level Differences from Hopfield and Other Existing Models or Mechanisms}

Hopfield networks are a classical paradigm in which dynamical states carry memories: recurrent dynamics converge to stored attractors, and retrieval is convergence \citep{Hopfield1982}. Modern Hopfield networks further establish a strict equivalence with Transformer attention: their single-step update rule is a self-attention layer in the adiabatic approximation \citep{ramsauer2020hopfield}. This equivalence means that architecture choice alone is insufficient to explain the difference between MetaPred and Transformer; the substantive distinction is at the paradigm level.

The key difference is the role of the state. In a Hopfield network, the state carries memory, and the dynamical objective is completely determined by a set of stored attractors; retrieval selects among known attractors. In MetaPred, the state carries both representation and prediction, and there is no predefined correct attractor. Its fixed-point loss does not require convergence to any stored pattern, but only self-consistency under one-step evolution. MetaPred therefore need not address the Hopfield tradeoff between storage capacity and retrieval accuracy or predefine a set of memory patterns.

From the perspective of training objectives, MetaPred's self-measurement fixed point can be mistaken for another fixed-point loss analogous to Hopfield energy descent. The key difference is that a Hopfield network reaches a local or global minimum through monotonic descent of a Lyapunov-style energy loss, whereas MetaPred uses a self-measurement signal to implement self-supervised learning through relative mean-squared error. A unique fixed point follows from Banach's theorem only when the complete parameter-group cycle has a Lipschitz constant $L<1$ on a closed set. This is a conditional conclusion and does not guarantee convergence throughout training. MetaPred's fixed-point loss is therefore an endogenous supervision signal required by paradigm symmetry and serves a different role from Hopfield energy.

Several developments in machine learning have increased the freedom of dynamical states to evolve on the model carrier.

First, inspired by global workspace theory in cognitive science, \citet{goyal2022workspace} proposed a shared workspace as a general mechanism for coordinating modules and demonstrated its effectiveness in Transformer architectures. This mechanism is highly similar to self-coordination in MetaPred.

Second, \citet{gu2022s4} proposed a sequence-modeling layer based on structured state spaces, replacing Transformer's quadratic-complexity attention mechanism with a linear time-invariant system. As a linear time-invariant system, its dynamics are independent of input content.

Third, \citet{dehghani2019universal} proposed the Universal Transformer, which implements deep recursion through parameter sharing across layers, supports dynamic computation, and achieves Turing completeness and stronger compositional generalization.

These methods expand dynamical freedom in module coordination, sequence modeling, and deep recursion, respectively, but all assume a particular model architecture. MetaPred instead implements dynamical evolution on a general SiFu topology. Apart from the external task objective, its own fixed-point supervision signal, and the SiFu topology, its dynamics are independent of additional constraints imposed by a specific model architecture.

Finally, compared with existing work, whether complete model designs or individual local machine-learning mechanisms, MetaPred is not an assembly of existing engineering techniques. It is a complete end-to-end model design supported by mathematical theory. These known local results did not jointly inspire MetaPred. Rather, the design and engineering implementation of MetaPred supported by AGI mathematical theory show, in turn, that these designs each have a reasonable place in the context of an AGI system.

\paragraph{Paradigm-Level Connections to Brain Function}

Implementing MetaPred requires a distance functional to distinguish state signals. Within a framework of global signal transforms, time-domain amplitude and frequency-domain phase form a dual pair of representations. We use time-domain amplitude for compatibility with existing computer hardware and software; the mathematical theory leaves the concrete form of the distance functional open. The brain instead uses a frequency-domain mechanism that distinguishes states through oscillatory phase.

This difference does not constitute a paradigm-level distinction. Electrophysiological evidence indicates that the phase-amplitude mechanism of brain oscillations implements predictive coding at the neuronal-population level: the phase of slower oscillations reflects the excitability of postsynaptic neurons and thereby synchronizes synaptic input associated with faster oscillations; gamma oscillations encode bottom-up prediction errors, while beta oscillations participate in top-down prediction \citep{noda2017phase}. Cross-frequency coupling links hierarchical interactions in predictive coding to specific frequency bands \citep{arnal2012cortical}. Cortical network models further show that inference in predictive coding naturally produces oscillatory dynamics \citep{bastos2012canonical}.

The self-measurement residual of MetaPred structurally corresponds to this prediction-error signal: a one-step simulation is a prediction, a $K$-step simulation is deeper inference, and their difference is the error. In the absence of external labels, this error signal still drives the dynamics toward internal consistency. This agrees with the central principle of predictive coding, in which prediction-error minimization provides a local learning signal, and with the free-energy account of the brain \citep{friston2010free}.

Therefore, setting aside the specific choice of distance functional, BriLLM-MetaPred is highly consistent with the brain in its operating mechanism: the same dynamical evolution simultaneously implements representation and prediction, and finite-depth self-comparison produces endogenous supervision. This consistency is not a design premise---BriLLM-MetaPred is designed entirely from mathematical theory. Its correspondence with known brain mechanisms is an independent observation that warrants further serious investigation.

\section{Notation and General Operators}
\label{sec:notation}
All core symbols, operations, and basic rules used in this paper are defined together here.

\subsection{Core Symbols and Indexing Rules}
All tensor dimensions are specified at first use, and indexing conventions are consistent throughout the paper:
\begin{itemize}

    \item \textbf{Nodes and neurons}: BriLLM-MetaPred is implemented as a neural network in which each node contains a group of neurons, with no independently trainable connections between neurons within a node. A node's signal is a $d$-dimensional real-valued vector, where $d$ is the node's neuron feature width. Edges between nodes are directed, and two edges in opposite directions may coexist between a pair of nodes. Signal propagation, transformation, and gating are implemented using standard neural operators, including fully connected layers, layer normalization, and GELU activation. Let the total number of nodes be $N=|V|$. The node signals form a tensor $\mathbf{S}\in\mathbb{R}^{N\times d}$.
    
    \item \textbf{Node partitions}: the node set $V$ is divided into 5 disjoint subsets. Subsequent operators act on these partitions by default:
    \begin{itemize}
        \item[] $I$: the input node region, which receives projected external token vectors;
        \item[] $O$: the output node region, which provides signals for prediction readout;
        \item[] $E$: the self-coordination node, which performs global signal coordination;
        \item[] $M$: the modal node region, which stores learned modal reference vectors;
        \item[] $W$: the working node region, which stores signals for dynamic computation;
        \item[] $U = M \cup W$: the adaptive internal node region;
        \item[] $C = \{E\} \cup M$: the coordination region, comprising the self-coordination node and the modal nodes;
        \item[] $F = I \cup W \cup O$: the auxiliary readout region, which provides pooled features for prediction readout.
    \end{itemize}
    
    \item \textbf{Graph topology constraints}: the model is defined on a directed graph with node set $V$. The graph is \textbf{fully connected in both directions}: edges exist between any two nodes, subject to the following restriction. Signal paths between the input region $I$ and output region $O$ pass through $M\cup W$ or $E$; this prevents input signals from taking a direct shortcut to the output readout. The implementation does not maintain a separate edge matrix for each node pair. Instead, communication between nodes is performed by three node-axis mixing operators acting on $I\cup W$, $I\cup C$, and $W\cup C\cup O$, respectively, each using two shared affine maps. This structure, with ``dense mixing within subsets and sequential connections across subsets,'' both ensures full connectivity in terms of reachability and limits the parameter count according to subset sizes rather than the square of the number of node pairs.

    \item \textbf{Vocabulary}: $\mathcal{V}$ denotes the vocabulary, and $|\mathcal{V}|$ its size.

    \item \textbf{Dynamical state}: $\mathbf{S}_{t,u} \in \mathbb{R}^{N\times d}$, where $N$ is the total number of nodes and $d$ is the node dimension. The external token index $t$ increases whenever 1 new token is written, and the internal executed evolution-step index $u$ increases whenever 1 standard evolution step is executed.
    
    \item \textbf{Parameter-group index}: $q_u = u \bmod Q$, where $Q$ is the total number of parameter groups. Each group contains independent evolution transformation weights and is reused cyclically. Parameter groups store only transformation parameters; node signals are stored entirely in $\mathbf{S}_{t,u}$. We abbreviate $q_u$ as $q$ below when this causes no ambiguity.
    
    \item \textbf{Three families of learnable parameters}: the model's learnable parameters are divided into three families by time scale:
    \begin{itemize}
        \item[] \textbf{Slow parameters} $\Slow$: shared across sequences and updated by batch. They include the input projection matrix $E_{\text{in}}$, output projection matrix $W_{\text{out}}, b_{\text{out}}$, slot projection vectors $\{s_k, p_k\}$, modal reference vectors $\{\mathbf m_k\}$, all weights in the five stages of single-step evolution, readout-path weights, slow prediction parameters $p_{\Slow}$, and slow correction parameters $c_{\Slow}$.
        
        \item[] \textbf{Fast parameters} $\Fast_u$: local to each sequence and updated in response to events within it. They include fast prediction parameters $p_{\Fast_u}$ and fast correction parameters $c_{\Fast_u}$, and are fully reset at sequence boundaries. The slow prediction parameters $p_{\Slow}$ and fast prediction parameters $p_{\Fast_u}$ are collectively called the state predictor.
        \item[] \textbf{Control parameters} $\boldsymbol{\psi}$: shared across sequences, updated by policy trajectory, and assigned to the controller network.
    \end{itemize}
    
    \item \textbf{Sequence runtime tuple}: $z_u = (\mathbf{S}_u, \Fast_u, \mathcal{H}_u, \mathcal{B}_u, q_u, \xi_u)$ is updated with every executed evolution step $u$:
    \begin{itemize}
        \item[] $\mathcal{H}_u$: a causal observation-history buffer that stores state snapshots, residuals, errors, and other metadata from executed steps, but no weights;
        \item[] $\mathcal{B}_u = (B_u^m, B_u^w)$: the remaining budgets for self-measurement and fast-parameter updates, respectively;
        \item[] $\xi_u$: the random-number generator state, which ensures complete isolation between the random streams for simulated evolution and actual state evolution;
    \end{itemize}
    We abbreviate $\mathbf{S}_{t,u}$ as $\mathbf{S}_u$ when this causes no ambiguity.
All runtime components are reset after a sequence has been processed, at the sequence boundary. Only $\Slow$ and the control parameters $\boldsymbol{\psi}$ are shared across sequences.
\end{itemize}

\subsection{General Operators}
The operations used to express the model are encapsulated in the following general operators:
\begin{itemize}
    \item \textbf{Fully connected layer}: $\operatorname{Linear}_{W,b}(x) = Wx + b \in \mathbb{R}^m$, where $W\in\mathbb{R}^{m\times n}$ is the weight matrix and $b\in\mathbb{R}^m$ is the bias vector.
    \item \textbf{Activation function}: we use the GELU activation function $\operatorname{GELU}(\cdot)$, which acts independently on each vector component.
    \item \textbf{Softmax}: for a vector $x\in\mathbb{R}^n$, the $i$th component of $\operatorname{softmax}(x)$ is $\exp(x_i)/\sum_j\exp(x_j)$.
    \item \textbf{Root mean square}: for a matrix $X$ or vector $x$, $\operatorname{RMS}(X) = \sqrt{\operatorname{mean}(X^2)}$.
    \item \textbf{Pooling operator}: $\operatorname{MeanPool}(G; \mathbf{S}) = \operatorname{mean}_G(\mathbf{S}) = \frac{1}{|G|}\sum_{n\in G} \mathbf{S}_n \in \mathbb{R}^d$ averages all node signals in a subset $G$ to extract global/local features.
    \item \textbf{Layer normalization}: $\operatorname{LN}(x) = \frac{x - \operatorname{mean}(x)}{\sqrt{\text{var}(x) + 10^{-6}}} \odot \gamma + \beta$ operates along the channel dimension, with learnable parameters $\gamma, \beta$.
    \item \textbf{Gated mixing}:
    \begin{equation}
     \operatorname{GateMix}(\mathbf{Z}, w, F) = (1-w) \odot \mathbf{Z} + w \odot F(\mathbf{Z}),
     \label{eq:gatemix}
    \end{equation}
    where $w \in [0,1]$ is an elementwise gate weight and $F$ is a signal transformation, implemented by default as $\operatorname{LN} \to \operatorname{Linear} \to \operatorname{GELU} \to \operatorname{Linear}$. This is the basic unit of all computational stages in MetaPred.
    \item \textbf{Bounding operator}:
    \begin{equation}
     \text{Bounded}(\mathbf{v}, \mathbf{r}, c) = \mathbf{v} \cdot \frac{c \cdot \text{RMS}(\mathbf{r})}{\text{RMS}(\mathbf{v}) + c \cdot \text{RMS}(\mathbf{r}) + 10^{-6}},
     \label{eq:bounded}
    \end{equation}
    smoothly rescales the magnitude of $\mathbf{v}$ in proportion to the RMS of a reference $\mathbf{r}$. The scaling factor decreases monotonically as $\text{RMS}(\mathbf{v})$ increases.
    \item \textbf{L2 normalization}: $\operatorname{norm}_2(x)=x/(\|x\|_2+\epsilon)$, where $\epsilon$ is a small constant.
    \item \textbf{Two clipping operations}:
    \begin{itemize}
        \item Norm clipping:
        \begin{equation}
         \operatorname{ClipNorm}(g, c) = g \cdot \min\left(1, \frac{c}{\|g\|_2 + \epsilon}\right),
         \label{eq:clipnorm}
        \end{equation}
        limits the norm of a gradient/update to at most $c$ without changing its direction, and is used for slow-gradient clipping and fast-parameter updates;
        \item Scalar clipping:
        \begin{equation}
         \operatorname{ClipTo}(\rho, a, b) = \min(\max(\rho, a), b),
         \label{eq:clipto}
        \end{equation}
        restricts a scalar to $[a,b]$ and is used to clip policy probability ratios.
    \end{itemize}
    
    \item \textbf{Normalized weighted mixing}:
    \begin{equation}
     \operatorname{WMix}(\{u_k\}, \{w_k\}) = \frac{\sum_k w_k u_k}{\sum_k w_k + 10^{-12}},
     \label{eq:wmix}
    \end{equation}
    computes a normalized weighted sum of vectors using nonnegative weights, and is used to aggregate modal features.
    
    \item \textbf{Gating map}:
    \begin{equation}
     \operatorname{Gate}(x; W, b) = \sigma(Wx + b),
     \label{eq:gate-map}
    \end{equation}
    where $\sigma$ is the componentwise sigmoid, used for coordination, broadcasting, readout, and related modules.
    \item \textbf{Relative MSE}:
    \begin{equation}
     \operatorname{RelMSE}(A,B) = \frac{\|A - \operatorname{sg}(B)\|_F^2 / n(A)}{\|\operatorname{sg}(B)\|_F^2 / n(A) + 10^{-6}},
     \label{eq:relmse}
    \end{equation}
    where $\operatorname{sg}(\cdot)$ is the stop-gradient operator and $n(A)$ is the number of scalar elements in $A$, preventing gradients from propagating through the target.
    
    \item \textbf{Measurement vector}: \begin{equation}
\begin{aligned}
\operatorname{MeasVec}(X) = [&\operatorname{mean}(X), \operatorname{RMS}(X), \operatorname{mean}|X-\operatorname{mean}(X)|,\\
&\min(X), \max(X), \operatorname{mean}\mathbf{1}[X\text{ finite}]],
\end{aligned}
\label{eq:measvec}
    \end{equation}
    maps an arbitrary tensor to 6 statistical features for trend supervision.
    \item \textbf{State summary}: for a state $\mathbf{S}$, define
\begin{equation}
s(\mathbf{S})=\operatorname{first}_{d_h}\!\left[
P_w\!\left(\operatorname{mean}_I(\mathbf{S})\right);
P_w\!\left(\operatorname{mean}_O(\mathbf{S})\right);
P_w(\mathbf{S}_E);
P_w\!\left(\operatorname{mean}_U(\mathbf{S})\right)
\right]\in\R^{d_h},
\label{eq:state-summary}
\end{equation}
where $P_w$ is a channel projection, $d_h$ is the summary dimension, and $\operatorname{first}_{d_h}$ selects the first $d_h$ components.
The controller uses a control state summary $s_c(\mathbf{S}_u)$, where the subscript $c$ indicates that the summary serves as an input to the controller. It is a different function from the summary $s(\mathbf{S})$ used for trend computation.
\item \textbf{Trend computation}: $\operatorname{Trend}(s, \Delta u, h) = \frac{s^{(h)}(\mathbf{S}') - s^{(h)}(\mathbf{S})}{\Delta u}$, where $s$ is the state summary defined in Eq.~(\ref{eq:state-summary}), $\mathbf{S}'$ is the state obtained after $\Delta u$ evolution steps from $\mathbf{S}$, $h=1$ denotes the first-order trend (velocity), $h=2$ denotes the second-order trend (acceleration), and $\Delta u$ is the evolution-step interval.
\end{itemize}

The BriLLM-MetaPred implementation contains two neural networks with independent parameters. The state network is implemented as a SiFu graph, contains the slow parameters $\Slow$ and fast parameters $\Fast_u$, and performs main state evolution. The controller network is a feedforward policy-value network that contains the control parameters $\boldsymbol{\psi}$ and configures runtime conditions. The two networks interact through the state tensor and runtime conditions.

\section{The Two Basic Dynamical Processes of MetaPred}

The dynamical construction of MetaPred contains only two basic processes: \textbf{main state evolution} and \textbf{self-measurement}. All other parts of the model modify or read from these two processes. Both processes share the same standard dynamics routine $\Phi_q^{\text{step}}$: main state evolution executes this routine on the runtime state and advances the evolution step $u$; self-measurement recursively executes the routine on a copy of the same state under controller scheduling, leaving the evolution step $u$ unchanged. Self-measurement thus forms a second-order recursive dynamical subprocess managed by the controller.
Main state evolution has two modes, depending on how external input is injected: \textbf{evolution with input} and \textbf{evolution without input}. Evolution with input is further divided into \textbf{encoding evolution} and \textbf{generation evolution}, depending on the content of the external input; evolution without input is \textbf{autonomous evolution}. Predictive readout reads directly from the current dynamical state without advancing the evolution step $u$ and is not part of the evolution process. The overall MetaPred model architecture is shown in Figure~\ref{fig:metapred-arch}.

\begin{figure}[htbp]
\centering\includegraphics[width=0.64\linewidth]{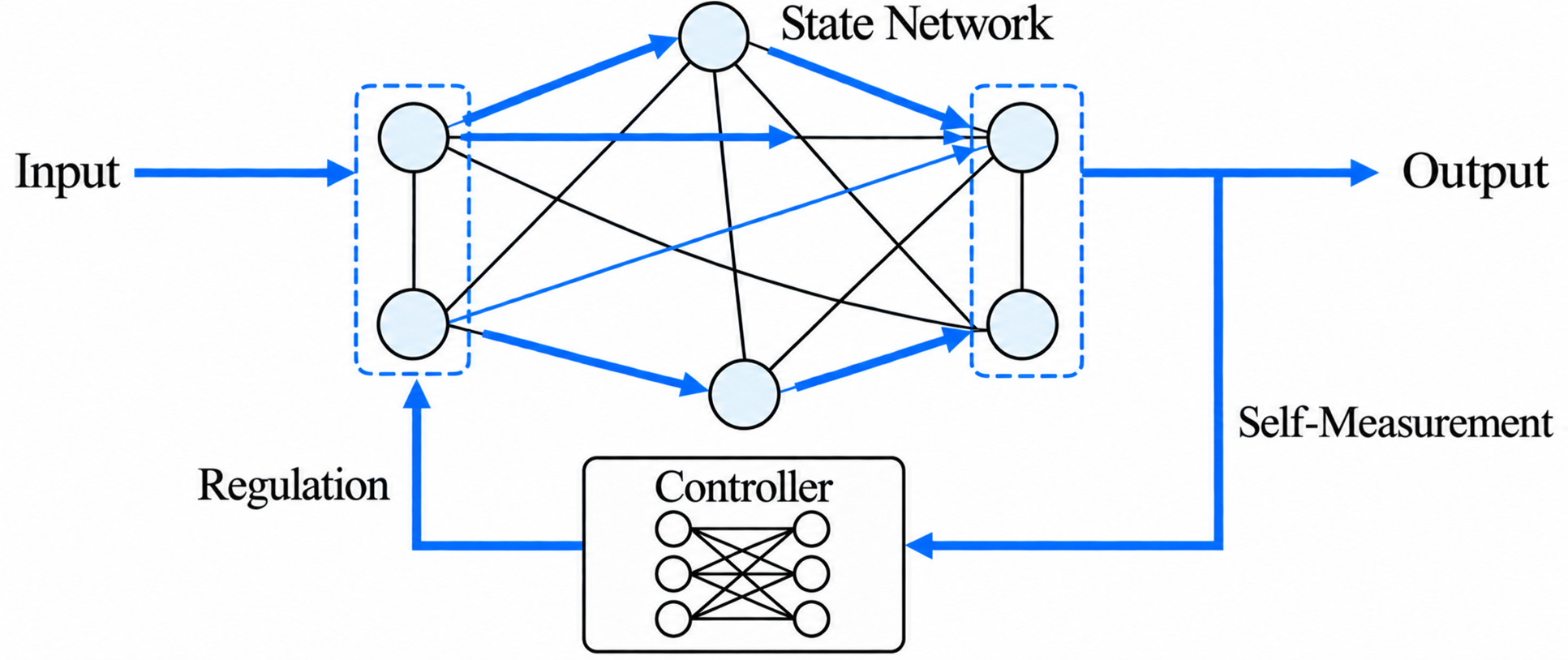}
\caption{The MetaPred model architecture}
\label{fig:metapred-arch}
\end{figure}

Based on these two basic processes, the complete execution flow of BriLLM-MetaPred can be divided into the following six layers:
\begin{itemize}
    \item \textbf{External interface layer}: the triggering conditions for main state evolution. It receives input tokens, writes their projection vectors into the input node region $I$, and triggers encoding evolution;
    \item \textbf{Core dynamics layer}: main state evolution itself. It updates the state through single-step evolution $\Phi_q^{\text{step}}$ and supports autonomous evolution and generation evolution;
    \item \textbf{Self-measurement layer}: self-measurement itself. It constructs two simulations with different numbers of steps from the same state, or provides an approximate residual through direct measurement, and computes the difference between the two as an endogenous supervision signal;
    \item \textbf{Decision layer}: a scheduling modification of the two basic processes. Based on the causal state summary and historical observations, the controller selects the measurement tier, correction strength, and fast-parameter update strategy for the current evolution step;
    \item \textbf{Readout layer}: a readout modification of main state evolution. It extracts features from the output region $O$ and coordination region $C$ of the dynamical state, adds a semantic residual, a structured residual, and slot alignment, and generates a vocabulary distribution through the output projection;
    \item \textbf{Update layer}: a learning modification of self-measurement. It separately updates the slow parameters $\Slow$, fast parameters $\Fast_u$, and control parameters $\boldsymbol{\psi}$, with fully isolated gradient paths.
\end{itemize}
Of these six layers, the external interface, core dynamics, self-measurement, readout, and update layers act directly on the same state tensor and share a single evolution path. The decision layer is implemented by an independent controller network: the controller reads a causal state summary from the state tensor and auxiliary metadata and outputs actions that affect the conditions for subsequent evolution.

The dynamical evolution process of MetaPred is shown in Figure~\ref{fig:metapred-dyn}.

\begin{figure}[htbp]
\centering\includegraphics[width=\linewidth]{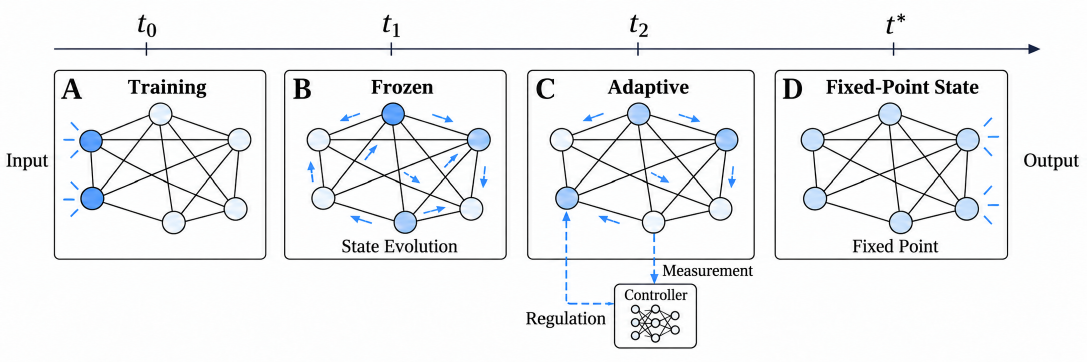}
\caption{Dynamical evolution in MetaPred}
\label{fig:metapred-dyn}
\end{figure}

The way MetaPred operates is not freely selectable; its three operating modes follow directly from the single dynamical state and the heterogeneity of parameter timescales. The three modes share the same evolution structure and differ only in the operations and parameter updates they permit. Table~\ref{tab:run-modes} compares the three operating modes. Evaluation metrics can be computed in all three modes, but they are not used for parameter updates in the frozen or adaptive modes.
\begin{table}[htp]
\caption{Comparison of the three operating modes}
\label{tab:run-modes}
\centering\small
\setlength{\tabcolsep}{5pt}
\begin{tabular}{>{\raggedright\arraybackslash}p{0.34\linewidth}>{\centering\arraybackslash}p{0.24\linewidth}>{\centering\arraybackslash}p{0.16\linewidth}>{\centering\arraybackslash}p{0.16\linewidth}}
\toprule
Feature & \textsc{Train} & \textsc{Frozen} & \textsc{Adaptive} \\
\midrule
External task loss used for parameter updates & Enabled & Disabled & Disabled   \\
Self-measurement diagnostics & Enabled & Enabled & Enabled \\
State correction & Enabled & Disabled & Enabled \\
Fast-parameter updates & Enabled & Disabled & Enabled \\
Slow-parameter updates & Enabled (per batch) & Disabled & Disabled \\
Control-parameter updates & Enabled (delayed PPO updates) & Disabled & Disabled \\
Sequence runtime reset & Sequence boundaries & Sequence boundaries & Sequence boundaries \\
\bottomrule
\end{tabular}
\end{table}

The three modes serve different purposes:
\begin{itemize}
    \item[] \textsc{Train} mode is used for training: the external task loss, self-measurement, fast-parameter updates, state correction, slow-parameter updates, and control-parameter updates are all enabled.
    \item[] \textsc{Frozen} mode is used for evaluation and self-diagnosis: all parameters are frozen and state correction is disabled. The external task loss can be computed on the validation set to monitor model capability, while self-measurement is executed in full and outputs residual signals. Because self-measurement requires only state evolution itself, the residual in this mode directly measures the model's internal consistency.
    \item[] \textsc{Adaptive} mode simulates real deployment: relying only on internal signals, it permits sequence-local fast-parameter updates and state correction while keeping the slow parameters $\Slow$ and control parameters $\boldsymbol{\psi}$ unchanged. All three modes execute dynamical evolution and self-measurement in full, differing only in the scope of permitted updates and corrections. All sequence runtime states are reset at sequence boundaries; only $\Slow$ and $\boldsymbol{\psi}$ are shared across sequences.
\end{itemize}

This division into three modes follows directly from two theoretical constraints of MetaPred, rather than from engineering optimization choices. First, a single state requires training and inference to share the same evolution structure, so all three modes must use the same dynamics. Second, endogenous supervision and training-inference coupling require self-measurement to be executed in full in every mode, so differences between the three modes are restricted to which parameter updates and state corrections are enabled.
When subsequent sections discuss parameter updates or state correction, they specify the operating modes in which the operation applies. The next two subsections define the two basic processes in order: main state evolution and self-measurement.

\subsection{Main State Evolution}
\label{sec:main-evolution}

All machine-learning operations of the state network are based on evolution over the main state space, which is a standard signal-dynamics routine.

\textbf{Single-step evolution}, $\Phi_q^{\text{step}}(\cdot; \chi): \mathbb{R}^{N\times d} \to \mathbb{R}^{N\times d}$, consists of five sequential gated-mixing stages. Here $\chi$ is the set of historical-dependency conditions for the evolution step, including runtime conditions such as the parameter-group index, invocation stage, input-window contents, and remaining budgets. The five stages share the same pattern---gated mixing, as defined in Eq.~(\ref{eq:gatemix}). Their mathematical definitions follow.

\begin{enumerate}
    \item \textbf{Input mixing.} Apply the node-axis mixing operator to three node subsets in sequence:
    \begin{equation}
     \mathcal{M}_G^q(Z)=Z+B_G^q\!\left(\operatorname{GELU}\!\left(A_G^q(\operatorname{LN}(Z))\right)\right),
     \label{eq:node-mixer}
    \end{equation}
    obtaining $Z^3=\mathcal{M}_{W\cup C\cup O}^q(\mathcal{M}_{I\cup C}^q(\mathcal{M}_{I\cup W}^q(\mathbf{S}+e_{\text{stage}})))$. None of the node subsets on which these three node-axis mixing operators act includes a direct connection between $I$ and $O$. The signal path from the input region to the output region therefore passes through $M\cup W$ or $E$.
    The operators $A_G^q,B_G^q$ are two fully connected layers along the node axis.

    \item \textbf{Role assignment.} For each internal node $n\in U$, compute four-dimensional role weights:
   \[
     \pi_n=\operatorname{softmax}\!\left(\frac{a_n^{\rm role}+\Delta_n}{\tau}\right),\qquad
     \Delta_n=\alpha\sigma(\gamma)c([Z^3_n,Z^3_E])+b_{n,q}.
    \]
    Here $a_n^{\rm role}\in\mathbb{R}^4$ is a learnable role prior for each node and parameter group; $\alpha,\gamma\in\mathbb{R}$ are learnable scalars; $\sigma$ is the sigmoid function; $c$ is a shared two-layer mapping $\mathrm{Linear}(2d,2d)\to\mathrm{GELU}\to\mathrm{Linear}(2d,4)$; and $\tau$ is a constant temperature set to $\tau=1$. The learnable bias $b_{n,q}\in\mathbb{R}^4$ is specific to each node and parameter group, with one $\mathbb{R}^4$ bias for every node in every parameter group.

    The four components are working, modal, fast, and slow, in that order. The working and modal components mix two shared channel transformations:
    \[
     \bar Z_n=Z^3_n+\frac{\pi_n^{\rm w}}{\pi_n^{\rm w}+\pi_n^{\rm m}}F^q_{\rm w}(Z^3_n)
     +\frac{\pi_n^{\rm m}}{\pi_n^{\rm w}+\pi_n^{\rm m}}F^q_{\rm m}(Z^3_n),
    \]
    where $F^q_{\rm w},F^q_{\rm m}$ are two shared channel transformations with the architecture $\operatorname{LN}\to\operatorname{Linear}\to\operatorname{GELU}\to\operatorname{Linear}$.

    \item \textbf{Coordination aggregation.} The self-coordination node $E$ compresses global information into $Z_E^c$:
    \begin{align}
     \mathbf m_k^*&=(1-\alpha_m)\mathbf m_k+\alpha_m\operatorname{norm}_2\!\left(\sum_i\omega_{k,i}\bar Z_i\right),\nonumber\\
     \omega_{k,i}&=\operatorname{softmax}_i(\mathbf m_k^\top\bar Z_i),\nonumber\\
     Z_E^I&=\bar Z_E+g_E(\bar Z_EW_E-\bar Z_E),\nonumber\\
     c_E&=\sum_k\operatorname{softmax}_k\bigl((Z_E^I)^\top\mathbf m_k^*\bigr)\mathbf m_k^*,\nonumber\\
     Z_E^c&=\bar Z_E+\operatorname{bounded}\bigl((Z_E^I-\bar Z_E)+g_c c_E,\,Z_E^I,\,c_E^{\max}\bigr).
     \label{eq:coordination}
    \end{align}
    Here $\{\mathbf m_k\}_{k=1}^{|M|}$ are learnable modal reference vectors among the slow parameters; $\mathbf m_k^*$ is the dynamic reference used for the current coordination operation, while the slow parameter $\mathbf m_k$ retains its original value; $\alpha_m,g_E,g_c$ are learnable scalars; and $c_E^{\max}$ is an upper-bound constant.

    \item \textbf{Broadcast distribution.} The self-coordination node distributes a broadcast offset to each $n\in U$:
    \begin{align}
     m_U&=\operatorname{WMix}\bigl(\{\bar Z_i\}_{i\in U},\{\pi_i^{\rm m}\}_{i\in U}\bigr),\nonumber\\
     b_n^{\rm br}&=\tanh\!\left(W_b[\operatorname{LN}([Z_E^c,m_U]);\operatorname{LN}(\bar Z_n)]\right),\nonumber\\
     \delta_n^b&=\operatorname{bounded}\!\left(\kappa_b\sigma(\gamma_b)c_b\pi_n^{\rm w}b_n^{\rm br},\,b_n^{\rm br},\,c_b^{\max}\right),\nonumber\\
     Z_n^b&=\bar Z_n+\delta_n^b.
     \label{eq:broadcast}
    \end{align}
    Here $\kappa_b,c_b,c_b^{\max}$ are the broadcast strength, control, and upper-bound constants, respectively. The broadcast targets the nodes in $U$; $I,O,E$ retain their original values.

    \item \textbf{Gated writeback.} Each node writes its intermediate signal back into the state through fast and slow gates:
    \begin{align}
     h_n&=\operatorname{Gate}\!\left([\mathbf{S}_n,\widehat Z_n];G_{\rm f}\right)\rho_n^{\rm f}
         +\operatorname{Gate}\!\left([\mathbf{S}_n,\widehat Z_n];G_{\rm s}\right)\rho_n^{\rm s},\nonumber\\
     \mathbf{S}'_n&=(1-h_n)\mathbf{S}_n+h_n\widehat Z_n,
     \label{eq:gate}
    \end{align}
    where $G_{\rm f},G_{\rm s}$ are two fully connected layers; $\rho_n^{\rm f}=\pi_n^{\rm f}$ and $\rho_n^{\rm s}=\pi_n^{\rm s}$ are the fast and slow components, respectively, of the four-dimensional role weights from the role-assignment stage; and $\widehat Z_n$ is the intermediate node signal after the first four stages.
    The output $\mathbf{S}'=\Phi_q^{\text{step}}(\mathbf{S}; \chi)$ is the result of one evolution step.
\end{enumerate}

Single-step evolution updates only the node signals $\mathbf{S}$, leaving the parameter-group index $q$, slow parameters $\Slow$, and history buffer $\mathcal{H}_u$ unchanged. Each evolution step injects a stage projection vector $e_{\text{stage}}$ according to the invocation stage (encoding, autonomous evolution, or generation) to distinguish the current stage of evolution. The vector $e_{\rm stage}\in\mathbb{R}^d$ is one of three learnable vectors corresponding to encoding, autonomous evolution, and generation; the same $e_{\rm stage}$ is added to each node during input mixing. State readout is an operation independent of evolution: it reads the predictive distribution directly from the current dynamical state without advancing the evolution step $u$.

The single-step evolution above is a concrete instantiation of the MetaPred paradigm on the node graph, with a structure determined by the five-stage signal lifecycle. These five stages correspond exactly to the complete lifecycle of a signal, from input, through internal coordination, to output:
input mixing ensures that external token information is incorporated into the current state;
role assignment allows each node to dynamically determine its emphasis among the four functions of ``working/modal/fast/slow,'' enabling flexible reorganization of the representation;
coordination aggregation compresses the global context through the self-coordination node, providing a consistent ``attention'' anchor for the entire state;
broadcast distribution returns the coordination signal to the internal nodes, making global information local;
gated writeback controls the magnitude of state updates through fast and slow pathways, ensuring the stability and correctability of evolution.
All five stages are based on a unified gated-mixing operator. They embody MetaPred's core principle of a ``single continuously evolving dynamical state'': rather than using a model design with separate encoders or decoders, all operations iteratively transform the same state tensor along a unified evolution path.

Evolution over $K$ consecutive steps is defined as a composition across parameter groups, termed \textbf{multi-step evolution}:
    \[
    \Phi_q^{(K)} = \Phi_{q+K-1}^{\text{step}} \circ \Phi_{q+K-2}^{\text{step}} \circ \dots \circ \Phi_q^{\text{step}}
    \]
When $K=Q+1$, $\Phi_q^{(Q+1)}$ starts at phase $q$ and executes $Q+1$ steps to reach phase $q+1$; equivalently, starting at phase $q+1$, the $Q$-step composition $\Phi_{q+1}^{(Q)} = \Phi_{q+Q}\circ\cdots\circ\Phi_{q+1}$ forms a complete parameter-group cycle, denoted by $\mathcal{C}_{q+1} = \Phi_{q+1}^{(Q)}$, corresponding to simulation-based measurement with $K=Q+1$. Multi-step evolution leaves the actual evolution-step counter $u$ unchanged and is used only for simulation and self-measurement.

All actual state evolution (encoding, autonomous evolution, and generation) advances $u$; all simulated evolution (self-measurement, previews, and alternative evaluation) leaves $u$ unchanged, ensuring complete isolation between runtime states and simulated states. Self-measurement belongs to the latter category: it recursively reuses the same standard dynamics routine $\Phi_q^{\text{step}}$, with its invocation timing, simulation depth, and measurement tier determined by the controller.

\subsection{Self-Measurement}
\label{sec:self-measurement}
This subsection defines the second basic process of MetaPred: self-measurement, a second-order recursive dynamical subprocess managed by the controller. It shares the same standard dynamics routine $\Phi_q^{\text{step}}$ with main state evolution. The difference is that main state evolution executes the routine on the runtime state and advances the evolution step $u$, whereas self-measurement recursively executes the routine on a copy of the same state under controller scheduling, leaving the evolution step $u$ unchanged. Self-measurement is therefore a nested process built on main state evolution, rather than being a static loss term independent of the dynamics.

The basic idea of self-measurement is to construct two simulations with different numbers of steps from the same state at an executed evolution step, using the difference between their endpoints as a signal. The 1-step simulation represents the model's direct prediction of the next step, whereas the $K$-step simulation represents the accumulated outcome of repeated dynamical evolution. Their difference measures the finite-depth inconsistency of the dynamics. This signal has four uses: training the shared dynamics, supervising direct measurement, driving state correction, and providing internal-consistency diagnostics. The first three uses correspond to parameter updates in \textsc{Train} or \textsc{Adaptive} mode; the fourth does not depend on any parameter update and serves as a self-diagnostic reading in \textsc{Frozen} mode.

The controller can choose among four measurement tiers: skipping measurement, direct measurement, short-horizon simulation, and long-horizon simulation (the latter two subdivide simulation-based measurement by depth). Skipping measurement produces a zero signal; direct measurement provides an approximate residual using only the direct-measurement parameters; simulation-based measurement explicitly simulates the state, with its depth determined by $K$. The direct-measurement parameters consist of the slow prediction parameters $p_{\Slow}$ and slow correction parameters $c_{\Slow}$ among the slow parameters, and the fast prediction parameters $p_{\Fast_u}$ and fast correction parameters $c_{\Fast_u}$ among the fast parameters.
When $K=2$ (termed short-horizon simulation), the simulation supervises direct measurement; when $K=Q+1$ (termed long-horizon simulation), the simulation corresponds to a complete parameter-group cycle and is used for deeper diagnostics. Short-horizon and long-horizon simulations use the same mechanism and differ only in the value of $K$.

Because simulating state evolution requires additional computation, the model also learns direct-measurement parameters, with the controller allocating measurement and update budgets. This subsection explains how these mechanisms work together in the execution order of a single internal evolution call. At the beginning of the call, the controller determines the measurement tier and update requests from the current state and historical observations; the residual is then obtained according to the selected tier. If a fast-parameter update is requested, direct measurement must be recomputed using the updated parameters. The model stores its prediction of the next state and executes one actual evolution step; it then applies state correction as needed and updates to the next state. Once the actual outcome is available, the model records the prediction error and measurement information to provide feedback for subsequent calls.

\paragraph{Simulation-based measurement}

Let $\chi_u$ denote the instance of the conditions at evolution step $u$, and let $\chi_u^{\text{meas}}$ and $\chi_u^{\text{real}}$ denote the conditions used for simulation-based measurement and actual state evolution, respectively. At the executed evolution step $u$, hold $\chi_u^{\text{meas}}$ fixed and construct two deterministic simulations:
    \begin{align*}
        \mathbf{S}_{\text{pred}} &= \Phi_q^{\text{step}}(\mathbf{S}_u; \chi_u^{\text{meas}}) \tag{1-step simulation} \\
        \mathbf{S}_{\text{sim}}^{(K)} &= \Phi_q^{(K)}(\mathbf{S}_u; \chi_u^{\text{meas}}) \tag{$K$-step simulation}
    \end{align*}
Their difference is the self-measurement residual:
    \[
    \Delta_{\text{self}}^{(K)}(u) = \mathbf{S}_{\text{sim}}^{(K)} - \mathbf{S}_{\text{pred}}.
    \]
Here $\Phi_q^{\text{step}}$ and $\Phi_q^{(K)}$ both use the same standard dynamics routine. Self-measurement introduces no new dynamics operator; it only executes the same dynamics recursively.

The residual has the termwise decomposition
\begin{equation*}
 \Delta_{\text{self}}^{(K)}(u)=\sum_{k=1}^{K-1}\left(\mathbf{S}_{\text{sim}}^{(k+1)}-\mathbf{S}_{\text{sim}}^{(k)}\right),\qquad
 \|\Delta_{\text{self}}^{(K)}(u)\|_F\leq\sum_{k=1}^{K-1}\|\mathbf{S}_{\text{sim}}^{(k+1)}-\mathbf{S}_{\text{sim}}^{(k)}\|_F,
\end{equation*}
where $\mathbf{S}_{\text{sim}}^{(1)}=\mathbf{S}_{\text{pred}}$.

When $K=Q+1$, the $K$-step simulation corresponds to a complete parameter-group cycle starting from $\mathbf{S}_{\text{pred}}$. Defining the composition starting at phase $q+1$ over $Q$ steps as $\mathcal{C}_{q+1} = \Phi_{q+1}^{(Q)}$, we have
\begin{equation*}
 \Delta_{\text{self}}^{(Q+1)}(u) = \mathcal{C}_{q+1}(\mathbf{S}_{\text{pred}}) - \mathbf{S}_{\text{pred}} = \Phi_q^{(Q+1)}(\mathbf{S}_u) - \mathbf{S}_{\text{pred}}.
\end{equation*}
The residual satisfies the telescoping decomposition property.

The self-measurement loss uses sequence-level normalization:
\begin{equation}
d_{\text{self},u} = \frac{\|\Delta_{\text{self}}^{(K)}(u)\|_F}{Nd}, \qquad
\mathcal{L}_{\text{self}} = \frac{1}{|\mathcal{D}_{\text{self}}|} \sum_{u\in\mathcal{D}_{\text{self}}} d_{\text{self},u}^2.
 \label{eq:l-self}
\end{equation}
Here $\mathcal{D}_{\text{self}}$ is the set of evolution steps at which simulation-based measurements have been executed; residuals with $K=Q+1$ do not participate in direct-measurement parameter calibration.

\paragraph{State correction based on self-measurement}

If the controller requests state correction, a correction term is added to the actual successor $\mathbf{S}_{u+1}^{\text{pre}} = \Phi_q^{\text{step}}(\mathbf{S}_u; \chi_u^{\text{real}})$:
\begin{equation}
\delta\mathbf{S}_u = \eta_{\text{eff}} \widetilde{\Delta}_{\text{self}}^{(K)}(u),
\qquad
\mathbf{S}_{u+1} = \mathbf{S}_{u+1}^{\text{pre}} + \delta\mathbf{S}_u.
\label{eq:correction}
\end{equation}
Here $\widetilde{\Delta}_{\text{self}}^{(K)}(u)$ is the available residual from simulation-based or direct measurement; $\eta_{\text{eff}}$ is the effective strength obtained by reducing the controller-requested strength according to amplitude limits.
Let controller action $a_u^{\rm corr}$ correspond to a nonnegative correction strength $\eta_u^{\text{req}}$, and let $\mathbf{S}_{u+1}^{\text{pre}}-\mathbf{S}_u$ be the state change due to actual state evolution. Define two amplitude limits:
\begin{itemize}
    \item[] $c_{\text{abs}}$: an absolute upper bound on the overall correction amplitude;
    \item[] $c_{\text{rel}}$: an upper bound on the ratio of the correction amplitude to the amplitude of the current actual evolution step.
\end{itemize}
Using $\operatorname{RMS}(\cdot)$ to measure the overall amplitude of a tensor and a small constant $\epsilon>0$ to avoid a zero denominator, the effective strength is
\begin{equation}
\eta_{\text{eff}} = \min\left\{
\eta_u^{\text{req}},
\frac{c_{\text{abs}}}{\operatorname{RMS}(\widetilde{\Delta}_{\text{self}}^{(K)}(u))+\epsilon},
\frac{c_{\text{rel}}\,\operatorname{RMS}(\mathbf{S}_{u+1}^{\text{pre}}-\mathbf{S}_u)}
     {\operatorname{RMS}(\widetilde{\Delta}_{\text{self}}^{(K)}(u))+\epsilon}
\right\}.
\label{eq:effective-correction-strength}
\end{equation}

These three values correspond to the controller's request, the absolute amplitude limit, and the relative amplitude limit, respectively; their minimum determines the actual correction. When no correction is requested, the operating mode prohibits correction, or the current residual is unavailable, the result of actual state evolution is used:
\begin{equation}
\mathbf{S}_{u+1} = \mathbf{S}_{u+1}^{\text{pre}}.
\label{eq:no-correction}
\end{equation}
The corrected state lies on the line segment between the two simulation endpoints (this holds only for simulation-based residuals).

\paragraph{Supervision of direct-measurement parameters}

Direct measurement uses the direct-measurement parameters to execute two prediction stages in sequence.

The first stage predicts the one-step successor before correction. Denoting the current state by $\mathbf{S}_u$ and the runtime conditions by $\chi_u$, we have
\[
\widehat{\mathbf{S}}_{u+1}^{\text{pre}}
= \mathbf{S}_u
+ p_{\Slow}(\mathbf{S}_u, \chi_u)
+ p_{\Fast_u}(\mathbf{S}_u).
\]
Here $p_{\Slow}$ is shared across sequences, whereas $p_{\Fast_u}$ is sequence-local.
In the implementation, the slow prediction mapping consists of a base prediction mapping and a slow prediction adapter:
\[
p_{\Slow}(\mathbf{S}_u,\chi_u)
=
p_{\Slow}^{\mathrm{base}}(\mathbf{S}_u,\chi_u)
+
p_{\Slow}^{\mathrm{adapt}}(\mathbf{S}_u).
\]
The base prediction mapping $p_{\Slow}^{\mathrm{base}}$ uses
$\mathrm{LN}\to\mathrm{Linear}(d,d)\to
\mathrm{GELU}\to\mathrm{Linear}(d,d)$.
Its input incorporates the stage, parameter group, and prediction horizon from the runtime conditions,
and its scaled output forms the base prediction increment.
The parameters of this mapping belong to the trainable slow parameters $\Slow$.

The slow prediction adapter $p_{\Slow}^{\mathrm{adapt}}$ and fast prediction mapping
$p_{\Fast_u}$ both use the bottleneck architecture
\[
\mathrm{LN}\to\mathrm{Linear}(d,r)\to
\mathrm{GELU}\to\mathrm{Linear}(r,d),
\]
where $r$ is the adapter bottleneck dimension.
These mappings operate on each node separately and share parameters across nodes; the mappings themselves perform no additional neighborhood aggregation along the node axis.
The slow prediction adapter uses slow parameters shared across sequences; the fast prediction mapping uses the fast parameters of the current sequence.
The complete state prediction is therefore obtained by adding the current state, base prediction increment, slow prediction adaptation increment, and fast prediction adaptation increment.

The second stage estimates the residual from the state predicted in the first stage:
\[
\widehat{\Delta}_{\text{self},u+1}^{(2)}
= c_{\Slow}(\widehat{\mathbf{S}}_{u+1}^{\text{pre}})
+ c_{\Fast_u}(\widehat{\mathbf{S}}_{u+1}^{\text{pre}}).
\]
Here both $c_{\Slow}$ and $c_{\Fast_u}$ take the predicted state $\widehat{\mathbf{S}}_{u+1}^{\mathrm{pre}}$ as input and produce correction components using the slow and fast parameters, respectively. Their sum forms the residual estimate from direct measurement.
Both mappings operate on each node separately, share parameters across nodes, and use
\[
\mathrm{LN}\to\mathrm{Linear}(d,r)\to
\mathrm{GELU}\to\mathrm{Linear}(r,d),
\]
producing a $d$-dimensional residual estimate at each node.
The mapping $c_{\Slow}$ uses shared slow parameters, whereas $c_{\Fast_u}$ uses sequence-local fast parameters.
Residual prediction does not directly modify the actual successor state $\mathbf{S}_{u+1}^{\mathrm{pre}}$. Once its output passes the direct-measurement availability check, it is used for state correction under the strength and amplitude constraints described above.

For observation $i$, let $\mathbf{S}_{\text{pred},i}$ denote its 1-step simulation and let its $K=2$ residual be $\Delta_{\text{self},i}^{(2)}$. Let $\mathcal{D}_{K=2}$ be the set of evolution steps at which $K=2$ simulation-based measurements have been executed. By the definition of $\mathcal{D}_{\text{self}}$, $\mathcal{D}_{K=2}\subseteq\mathcal{D}_{\text{self}}$. The loss for supervising the direct-measurement parameters is
\begin{equation}
    \mathcal{L}_{\text{direct-self}} = \mathbb{E}_{i\in\mathcal{D}_{K=2}} \frac{1}{2}\left[\operatorname{RelMSE}(\widehat{\mathbf{S}}_i^{\text{pre}}, \mathbf{S}_{\text{pred},i}) + \operatorname{RelMSE}(\widehat{\Delta}_{\text{self},i}^{(2)}, \Delta_{\text{self},i}^{(2)})\right]
\label{eq:l-directself}
\end{equation}
Only $K=2$ simulation-based measurements that are already available are used as supervision signals; residuals with $K=Q+1$ do not participate in direct-measurement parameter training.

The availability of direct-measurement parameters is determined by validation against previously recorded simulation-based measurements. The checks determine whether the invocation stage, parameter-group index, fast-parameter version, and simulation depth match the current request; whether the record remains within its validity window; and whether the estimation error is below a relative threshold. If any condition is not met, the direct measurement is marked unavailable, and the current correction uses a simulation-based residual or is skipped. Direct measurement with $K=2$ and simulation results with $K=Q+1$ belong to different measurement tiers; the $K=Q+1$ residual is obtained through the corresponding simulation-based measurement.

If the controller requests an update to the fast parameters $\Fast_u$ in the current call, this update changes the outputs of $p_{\Fast_u}$ and $c_{\Fast_u}$. Therefore, whenever direct measurement is selected, both prediction stages must be rerun after the $\Fast_u$ update is complete, followed by validation according to the rules above.
The simulation-based residual is independent of fast-parameter updates and retains its original value.

We emphasize that the amplitude of the direct-measurement parameter output and the amplitude of the actual residual are independent. Direct measurement only approximates the current residual, and its reliability is ensured by validation against simulation-based measurements. Regardless of whether the current call performs state correction, simulation-based measurement participates in internal-consistency learning and provides supervision for subsequent direct measurement; it is therefore indispensable to the complete training procedure.

The predicted state $\widehat{\mathbf{S}}_{u+1}^{\text{pre}}$ and actual state $\mathbf{S}_{u+1}^{\text{pre}}$ may differ: direct measurement predicts under the runtime conditions before the action, whereas actual state evolution produces the state under the runtime conditions after the action.

The self-measurement mechanism operates within a conditional-convergence framework.
Only when the complete parameter-group cycle $\mathcal{C}_{q+1} = \Phi_{q+1}^{(Q)}$, which depends on the slow parameters $\Slow$ and runtime conditions $\chi$, is a contraction on the closed subset $\mathcal{X}$,
\begin{equation*}
 \|\mathcal{C}_{q+1}(A) - \mathcal{C}_{q+1}(B)\|_F \leq L\|A-B\|_F,\quad L<1,
\end{equation*}
does the Banach fixed-point theorem guarantee a unique fixed point $\mathbf{S}^\star$ and
\begin{equation*}
 \|\mathcal{C}_{q+1}^m(\mathbf{S}) - \mathbf{S}^\star\|_F \leq L^m \|\mathbf{S} - \mathbf{S}^\star\|_F.
\end{equation*}
This result is conditional. The convergence of a trained model is determined by its specific parameters and the corresponding computational dynamics.

\section{BriLLM-MetaPred Operational Architecture}
\label{sec:metapred-framework}

This section describes the external interface, readout, and decision layers in MetaPred's six-layer structure. These three layers modify or read out the two basic processes: the external interface layer provides the triggering conditions for main state evolution, the readout layer modifies the readout of main state evolution, and the decision layer modifies the scheduling of the two basic processes.

\subsection{External Interface Layer: Token Input and State Writing}
External tokens are written into the dynamical state in three steps:
\begin{enumerate}
    \item \textbf{Projection lookup}: for an input token sequence $x_{1:T}$ ($x_t \in \mathcal{V}$), the shared input projection matrix $E_{\text{in}} \in \mathbb{R}^{|\mathcal{V}|\times d}$ gives the token projection vector $e_t = E_{\text{in}}[x_t]$. $E_{\text{in}}$ is randomly initialized using Xavier uniform initialization; the input projection matrix is a subset of the slow parameters $\Slow$.
    \item \textbf{Sliding-window update}: the input region $I$ maintains an ordered window of capacity $|I|$. When a new token projection vector arrives, the oldest projection vector is discarded from the left end and the new projection vector is appended to the right:
    \[
    H_t = (e_{t-|I|+1}, \dots, e_{t-1}, e_t)
    \]
    When $t<|I|$, the left end is padded with zero vectors; information from earlier tokens beyond the window capacity is compressed into the working region $W$ and modal region $M$ through dynamical evolution and retained as long-term memory.
        \item \textbf{State writing}: the window contents overwrite the input-region state slot by slot, with two learnable slot projection vectors $s_k$ and $p_k$ added:
    \[
    (\mathbf{S}^{\text{write}}_t)_{i_k} = e_{t-|I|+k} + s_k + p_k, \quad k=1,\dots,|I|
    \]
    Here $i_k$ is the index of the $k$-th node in the input region; $s_k$ and $p_k$ correspond to the slot embedding table and position embedding table in the implementation, respectively. Both are learnable parameters initialized from the normal distribution $\mathcal{N}(0, 0.02^2)$.
\end{enumerate}

\textbf{Evolution triggering rules}: state writing advances only the token step $t$, leaving the evolution step $u$ unchanged; after writing, $K_{\text{in}}$ steps of encoding evolution $\Phi_q^{\text{step}}$ are triggered. After the entire context has been written, $K_{\text{ev}}$ steps of autonomous evolution are triggered, with no new token injected. During generation, 1 step of generation evolution is triggered after each target token is read. Encoding evolution and generation evolution are both evolution with input, while autonomous evolution is evolution without input. The initial state of every sequence is a zero tensor, and the state is fully reset at sequence boundaries.

\subsection{Readout Layer: Prediction Readout}
The readout layer reads predictions directly from the dynamical state of the state network; the readout mapping is implemented by the output projection.

First, the $\operatorname{MeanPool}$ operator extracts pooled features from 5 node subsets:
    \begin{align}
     o = \operatorname{MeanPool}(O; \mathbf{S}),\quad i = \operatorname{MeanPool}(I; \mathbf{S}),\quad c = \operatorname{MeanPool}(C; \mathbf{S}),\\  \quad
     f = \operatorname{MeanPool}(F; \mathbf{S}),\quad m = \operatorname{MeanPool}(M; \mathbf{S}),
    \label{eq:pools}
    \end{align}
    Here $F = I\cup W\cup O$ is the auxiliary readout region.
    
Next, the complete readout feature combines the base path, structured residual, and slot alignment, with the semantic residual incorporated into the base path. All components come from the same dynamical state. They are constructed as follows:
    \begin{itemize}
        \item \textbf{Semantic residual}: the original token projection vectors in the input window $H_t = (e_{t-|I|+1},\ldots,e_t)$ provide a direct differentiable training path to $E_{\text{in}}$:
        \begin{align}
        q_{\rm sem}&=[o;i;c;e_t;\operatorname{mean}(e_{t-3:t});\operatorname{mean}(H_t)],
        &u_{\rm sem}&=F_{\rm sem}(\operatorname{LN}(q_{\rm sem})),\nonumber\\
        \delta_{\rm sem}&=\operatorname{Bounded}\left(u_{\rm sem}, u_{\rm sem}, \sigma(\gamma_{\rm sem}) c_{\rm read} c_{\rm sem}^{\max}\right).
        \label{eq:sem-clip}
        \end{align}
        Here $c_{\text{read}}$ is a scalar hyperparameter controlling the magnitude of the semantic residual. Together with $c_{\text{sem}}^{\max}$, it limits the residual scale to avoid interfering with the base readout path. $F_{\rm sem}$ is a two-layer fully connected mapping with input dimension $5d+d=6d$, hidden dimension $2d$, and output dimension $d$. The hidden layer uses GELU activation.
        
        \item \textbf{Base path}: pooled features are combined into the base readout feature $r_{\text{base}}$, with a gating mechanism dynamically weighting features from different sources:
        \[
r_{\text{base}} = \operatorname{LN}\left(r_0 + g_{fs} \odot W_f f + (1-g_{fs}) \odot W_c c + \delta_{\text{sem}}\right)
        \]
Here $r_0 = W_O o + b_O + g_I \odot (W_i i + b_i)$, and $g_I, g_{fs}$ are gating weights produced by the $\operatorname{Gate}$ operator, where
\[
g_I = \operatorname{Gate}([o;i;c]; W_I, b_I),\qquad
g_{fs} = \operatorname{Gate}([f;c;r_0]; W_{fs}, b_{fs}),
\]
$\operatorname{Gate}$ is defined in Eq.~(\ref{eq:gate-map}).
        
\item \textbf{Structured residual}:
      The historical observation features $\mu_u$ are provided by the sequence runtime and are distinct from the compatible measurement history $\mathcal{Q}_u$ in the controller input. Here $e_u^{\text{pred}}$ and $e_u^{\text{corr}}$ are the prediction error and correction error of the most recent matured observation, respectively:
\begin{equation*}
\mu_u=[e_u^{\text{pred}},e_u^{\text{corr}},\operatorname{RMS}(v_u),\operatorname{RMS}(a_u),
\|\Delta\Fast_u\|,\|\delta\mathbf{S}_u\|_{\operatorname{RMS}},
B_u^{m}/B_0^{m},B_u^{w}/B_0^{w}]\in\mathbb{R}^8.
\end{equation*}
        Learned output-node weights $a_m$, satisfying $\sum_l a_{m,l}=0$ and $\|a_m\|_2=1$, are combined with the historical observation features $\mu_u$ to supply a structured offset:
        \begin{align}
        z_m&=\sum_l a_{m,l}O_l,\nonumber\\
        \omega&=\operatorname{softmax}(W_\omega[o;\mathbf{S}_E;m;\mu_u]+b_\omega),\nonumber\\
        u_{\rm str}&=W_\uparrow\sum_s\omega_s\operatorname{GELU}(W_\downarrow\operatorname{LN}(q_s)),\nonumber\\
        \delta_{\rm str}&=\operatorname{Bounded}\left(u_{\rm str}, r_{\text{base}}, \sigma(\gamma_{\rm str}) c_{\rm read} \kappa_{\rm str}\right),
        \label{eq:str-clip}
        \end{align}
        Here $q_s\in\{z_1,\ldots,z_R,\mathbf{S}_E,m\}$. The vectors $z_1,\ldots,z_R$ are pooled subset features from the output region $O$ and auxiliary region $F$, with $R=4$, corresponding to linear projections of the four pooled vectors $O$, $I$, $F$, and $M$, respectively; they provide multiple reference directions for the structured residual.
        $a_m$ is derived from an unconstrained parameter $\tilde a_m\in\mathbb{R}^{|O|}$ by setting $a_m=\tilde a_m-\operatorname{mean}(\tilde a_m)$ and then dividing by $\|\tilde a_m-\operatorname{mean}(\tilde a_m)\|_2$.
        
        \item \textbf{Slot alignment}: for generation position $k_j = j-1$, an offset between the output node and pooled feature is added:
        \[
        \delta_{\text{slot},k_j} = \operatorname{Bounded}\left(\sigma(\gamma_{\text{slot}})(W_{\text{slot}}\operatorname{LN}(O_l - o) + b_{\text{slot}}), r_{\text{base}}+\delta_{\text{str}}, \kappa_{\text{slot}}\right)
        \]
        Here $O_l$ is the feature of the $l = k_j \bmod |O|$-th node in the output region.
    \end{itemize}
    
The final logits and probabilities are obtained from the base feature plus the two residuals:
    \begin{equation}
     r_j = r_{\text{base}} + \delta_{\text{str}} + \delta_{\text{slot},k_j}, \quad \ell_j = W_{\text{out}} r_j + b_{\text{out}}, \quad p_j = \operatorname{softmax}(\ell_j),
     \label{eq:logits}
    \end{equation}
    Here $W_{\text{out}}, b_{\text{out}}$ are independent output projection parameters belonging to the slow parameters $\Slow$.

\subsection{Decision Layer: The Controller}

The MetaPred model must contain two independent neural networks: a state network implemented by the SiFu graph, and a controller network that performs meta-state control. The state network is organized into the five node partitions defined above. Nodes communicate through node-axis mixing operators, while channel transformations process information within each node. Its parameters are the slow parameters $\Slow$ and fast parameters $\Fast_u$. The state network interacts with external inputs and outputs through the input projection matrix $E_{\text{in}}$ and output projection matrix $W_{\text{out}}$.

The controller network is a feedforward policy-value network whose parameters are independent of those of the state network, with an actor-critic architecture. It receives the causal state summary $o_u$ and processes it through learnable layer normalization and two fully connected hidden layers to obtain a shared representation of the same dimension:
\[
z_u=\tanh\!\left(W_2\tanh\!\left(W_1\operatorname{LN}(o_u)+b_1\right)+b_2\right),
\]
Here $\operatorname{LN}$ is learnable layer normalization. The shared representation $z_u$ feeds six action heads and one scalar value head. The value head outputs the scalar estimate $V_{\boldsymbol{\psi}}(o_u)$ for PPO training.

The controller network is called once before the start of each executed evolution step. It reads the causal state summary $o_u$ from the state tensor and auxiliary metadata, then outputs a set of measurement and update actions $a_u$ that configure the runtime conditions $\chi_u$ for subsequent evolution. The actions configure only the runtime conditions $\chi_u$; the state tensor and evolution-step counter remain unchanged. The controller schedules the self-measurement subprocess through the measurement actions $a_u^{\text{tier}}$, $a_u^{\text{focus}}$, and $a_u^{\text{corr}}$: $a_u^{\text{tier}}$ determines whether direct measurement or simulation-based measurement is performed and sets the simulation depth; $a_u^{\text{focus}}$ specifies the measurement object to focus on; and $a_u^{\text{corr}}$ determines the state-correction strength. Thus, the controller manages execution of the second-order recursive dynamical subprocess through a unified mechanism rather than a fixed procedure.
The timing and triggering conditions for updates to the control parameters $\boldsymbol{\psi}$ are handled by the PPO step specified in the model's training algorithm.

The controller's input and output data flow is as follows:

The controller reads the causal state summary $o_u$, which contains only information available from executed steps:
\[
o_u = \left[s_c(\mathbf{S}_u);\, r_u;\, \mathcal{Q}_u;\, \tau_u;\, \zeta_u;\, a_{u-1};\, \mathrm{stage}_u;\, b_u;\, d_u\right]
\]
where
\begin{itemize}
\item[] $s_c(\mathbf{S}_u)$: the control state summary, containing the global mean, RMS, RMS after centering along the node axis, the RMS values of input/output/internal/self-coordination nodes, and the proportion of finite values;
\item[] $r_u$: recent runtime signals, containing prediction error, intrinsic energy, the RMS values of velocity and acceleration, the norms of the fast-parameter update and state correction, the previous reward, and the parameter version;
\item[] $\mathcal{Q}_u$: compatible measurement history, containing difference statistics of existing compatible measurements and the corresponding validity information;
\item[] $\tau_u$: the normalized number of steps since the last measurement for each of the nine classes of measurement objects;
\item[] $\zeta_u$: whether each of the nine classes of measurement objects has no record yet;
\item[] $a_{u-1}$: the indices of the six previous action components;
\item[] $\mathrm{stage}_u$: a one-hot encoding of the invocation stage (encoding, autonomous, generation);
\item[] $b_u$: the remaining fractions of the measurement budget and fast-parameter update budget;
\item[] $d_u$: runtime diagnostics including phase, sample-pool and prediction-queue occupancy, first- and second-order energy differences and their validity, self-measurement residual and loss, and calibration status.
\end{itemize}
Gradients are stopped through all components fed into the controller.
    
The controller outputs an action $a_u$, sampled from six discrete action heads and written as
\[    a_u=\bigl(a_u^{\text{tier}},a_u^{\text{focus}},a_u^{\text{upd}},a_u^{\text{str}},a_u^{\text{corr}},a_u^{\text{ratio}}\bigr),
\]
where $a_u^{\text{tier}}$ is the measurement tier, $a_u^{\text{focus}}$ is the measurement object, $a_u^{\text{upd}}$ is the fast-parameter update switch, $a_u^{\text{str}}$ is the fast-parameter update strength, $a_u^{\text{corr}}$ is the state-correction strength, and $a_u^{\text{ratio}}$ is the sampling ratio of matured observations used to update fast parameters. All components use discrete tiers. Each of the six action heads produces a categorical distribution; invalid categories are masked before sampling from the distributions to form action $a_u$. The following conditions further override the sampled results with default actions: insufficient coverage, warm-up, budget exhaustion, uncalibrated direct-measurement parameters, and simulation-based measurement intervals. These override rules are hard implementation constraints and are not learned through the control parameters. When the update switch is off, the sampling-ratio action head index is set to zero; requesting an update and executing it are two different steps.
    
The controller is called in all three operating modes (see Table~\ref{tab:run-modes}), but only \textsc{Train} and \textsc{Adaptive} permit fast-parameter updates and state corrections. In \textsc{Frozen} mode, the controller still computes actions, whose effects are limited to configuring runtime conditions. The control parameters $\boldsymbol{\psi}$ are frozen in all evaluation modes.

\section{MetaPred Training Algorithm}
\label{sec:training-inference}

This section describes the update layer in MetaPred's six-layer structure, namely the learned refinement of the self-measurement process. The three parameter families are updated separately according to their time scales: slow parameters $\Slow$ are shared across sequences and updated by batch, fast parameters $\Fast_u$ are local to each sequence and updated by events within the sequence, and control parameters $\boldsymbol{\psi}$ are shared across sequences and updated from policy trajectories.

MetaPred organizes training signals differently from layer-by-layer backpropagation in the conventional neural-network sense. The latter implicitly assumes three conditions: the forward and backward computations share the same layer structure, the loss is a single terminal function, and the backward computation within a single weight-update cycle proceeds layer by layer in layer order along a unique path. MetaPred uses a different organization for all three conditions. Gradient computation itself still uses standard automatic differentiation; the differences lie in the sources of training signals and the organization of parameters.

First, the SiFu graph has no layer hierarchy, and its parameters are partitioned by functional role.

Second, training signals consist of multiple self-comparisons of the same state from different perspectives: $\mathcal{L}_{\text{pred}}$ compares the state with its successor, $\mathcal{L}_{\text{fp}}$ compares the state with a one-step replay, $\mathcal{L}_{\text{self}}$ compares a 1-step simulation with a $K$-step simulation, and $\mathcal{L}_{\text{direct-self}}$ compares predicted and simulated residuals.

Third, the training weight updates for the main component of the MetaPred model, the state network, are completed synchronously in a single operation within each update cycle: automatic differentiation computes the gradients over the state tensor in a single computation, after which the parameters grouped by function are updated synchronously in a single operation. Unlike conventional hierarchical neural networks, there is no backpropagation path that follows the layer order; instead, backward paths from multiple observation points converge on the same graph without a layer hierarchy, rather than propagating backward layer by layer.

This organization is a necessary choice under the mathematical theory of AGI. The single-state constraint requires a unified encoder--decoder path so that representation and prediction share their implementation mechanism and parameters; the endogenous-supervision constraint requires losses to come from comparisons of the state with itself; the training--inference coupling constraint requires the training path to remain present during inference; and the functional heterogeneity of parameters requires general dynamics, sequence-local adaptation, and resource-allocation policies to be updated separately on their respective time scales.
Together, these four constraints leave only one workable organization: multiple self-comparisons on a single state, parameter updates grouped by function, and a shared state-evolution structure for training and inference. This chain of reasoning is coherent: the six-layer decomposition of the architecture is determined by the functional requirements of the MetaPred primitive, the three-way separation of the training algorithm is determined by the different parameter time scales, and the two designs are unified through the same dynamical structure. The training and inference procedures described below are both implemented using the standard dynamics routine.

\subsection{Multi-task Training Objective}
The training objective consists of an external task loss and endogenous self-supervised losses, with all endogenous self-supervised losses using the $\operatorname{RelMSE}$ operator:
\begin{equation}
 \boxed{
 \mathcal{L}_{\text{outer}} = \mathcal{L}_{\text{task}} + \lambda_{\text{int}}\left(\lambda_p\mathcal{L}_{\text{pred}} + \lambda_f\mathcal{L}_{\text{fp}} + \lambda_d\mathcal{L}_{\text{dyn}} + \lambda_c\mathcal{L}_{\text{direct-self}}\right) + \lambda_s\mathcal{L}_{\text{self}}
 }
 \label{eq:outer-loss}
\end{equation}
The per-observation state prediction loss $\mathcal{L}_{\text{pred},i}$ is defined below under the state prediction loss term; because the target $\mathbf{S}_i^{\rm pre}$ is detached, gradients of this loss flow only through the state predictors $p_{\Slow}, p_{\Fast_u}$.
Let $\mathcal{D}_{\text{pred}}$ denote the set of mature state-prediction observations, $\mathcal{D}_{\text{fp}}$ the set of observations whose pre-correction states have been reached, and $\mathcal{D}_{\text{dyn}}$ the set of observations with available dynamical trends. The definitions and purposes of the loss terms are as follows:
\begin{itemize}
    \item \textbf{External task loss}: $\mathcal{L}_{\text{task}} = -\frac{1}{BH}\sum_{b=1}^B\sum_{j=1}^H \log p_{b,j}[y_{b,j}]$ is the next-token cross-entropy, available only when external labels are provided.
    
    \item \textbf{State prediction loss}: the per-observation loss is
\[
\mathcal{L}_{\text{pred},i} = \operatorname{RelMSE}(\widehat{\mathbf{S}}_i^{\text{pre}}, \mathbf{S}_i^{\text{pre}}),
\]
and the overall loss is
\[
\mathcal{L}_{\text{pred}} = \mathbb{E}_{i\in\mathcal{D}_{\text{pred}}} \mathcal{L}_{\text{pred},i},
\]
which monitors the accuracy of the state predictors in predicting the actual one-step successor. The target $\mathbf{S}_i^{\text{pre}}$ is detached.
    
    \item \textbf{Fixed-point proxy loss}:
    \begin{equation}
     \mathcal{L}_{\text{fp}} = \frac{1}{|\mathcal{D}_{\text{fp}}|}\sum_{i\in\mathcal{D}_{\text{fp}}} \operatorname{RelMSE}\left(\Phi_{q_i}^{\text{step}}(\mathbf{S}_i^{\text{pre}}; \chi_i), \mathbf{S}_i^{\text{pre}}\right),
     \label{eq:fp-proxy}
    \end{equation}
    A single ordinary evolution call is replayed using the stored parameter-group index $q_i$ and historical conditions $\chi_i$ to test whether the pre-correction state actually reached remains unchanged after one evolution step; this does not advance the actual evolution-step counter. This probe differs from simulation-based measurement in that it acts on the pre-correction state \textbf{already reached}, whereas both paths of simulation-based measurement start from the same \textbf{current} state.
    
    \item \textbf{Dynamical trend loss}: for each mature observation $i$, define
    \[
    \mathcal{L}_{\text{dyn},i} = \frac{w_s\mathcal{L}_{\text{state},i} + I_i^m w_m\mathcal{L}_{\text{measurement},i} + I_i^e w_e\mathcal{L}_{\text{energy},i}}{w_s + I_i^m w_m + I_i^e w_e},
    \]
    and set $\mathcal{L}_{\text{dyn}}=\mathbb{E}_{i\in\mathcal{D}_{\text{dyn}}}\mathcal{L}_{\text{dyn},i}$ to supervise the consistency of state trends, measurement-statistic trends, and energy trends, using only compatible predecessors of mature observations for the computation. Here, $I_i^m,I_i^e\in\{0,1\}$ indicate whether the measurement trend and energy trend, respectively, are available for observation $i$; $\mathcal{L}_{\text{state}}, \mathcal{L}_{\text{measurement}}, \mathcal{L}_{\text{energy}}$ are the RelMSE losses for the state trend, measurement-statistic trend, and energy trend, respectively; and $\mathcal{D}_{\text{dyn}}$ is the set of observations with available dynamical trends.

    Let $s(\mathbf{S})$ be the state summary, and let $u_i^0,u_i^1$ be the evolution-step indices at which observation $i$ is issued and its target is reached, with $\Delta u_i=u_i^1-u_i^0\geq1$. Define the order-$h$ trend as
\begin{equation}
v_i^{(h)}=\frac{s^{(h)}(\mathbf{S}_i^{\text{pre}})-s^{(h)}(\mathbf{S}_i)}{\Delta u_i},\qquad h\in\{1,2\},
\end{equation}
where $s^{(1)}=s$, and $s^{(2)}$ is the finite difference of the first-order trend along evolution steps, namely the second-order trend (acceleration).
The state trend head $T_{\Slow}$ takes as input the state summary $s(\mathbf{S}_i)\in\mathbb{R}^{d_h}$ at the time the observation is issued, and outputs both a first-order trend correction and a second-order trend prediction:
\[
\left[b_i^{(1)},\widehat v_i^{(2)}\right]
= T_{\Slow}\!\left(s(\mathbf{S}_i)\right) \in\mathbb{R}^{2d_h},
\]
where both output components belong to $\mathbb{R}^{d_h}$.
Let $h_T$ be the hidden dimension of this trend head, with $h_T\ge 2d_h$. Its architecture is
\[
\mathrm{LN}\to\mathrm{Linear}(d_h,h_T)\to
\mathrm{GELU}\to\mathrm{Linear}(h_T,2d_h).
\]
The first-order trend prediction is obtained by adding the correction from the trend head to the rate of change implied by the state prediction:
\[
\widehat v_i^{(1)}
=
\frac{s(\widehat{\mathbf{S}}_i^{\mathrm{pre}}) - s(\mathbf{S}_i)}{\Delta u_i}
+ b_i^{(1)}.
\]
The second-order trend prediction $\widehat v_i^{(2)}$ is produced directly by the trend head.
The per-observation state trend loss is
\begin{equation}
\mathcal{L}_{\text{state},i}=\frac{\sum_{h\in\mathcal{A}_i^v}\operatorname{RelMSE}(\widehat v_i^{(h)},v_i^{(h)})}{|\mathcal{A}_i^v|},
\end{equation}
where $\mathcal{A}_i^v\subseteq\{1,2\}$ is the set of available orders; $\widehat v_i^{(1)}$ is the combined prediction above, and $\widehat v_i^{(2)}$ is the second-order prediction from the trend head.
The targets for all orders are detached, and the supervised loss is computed only for valid mature observations.

The input to the measurement trend head concatenates the $d_h$-dimensional state summary with a one-hot encoding of the nine classes of measurement objects, giving an input dimension of $d_h+9$; its hidden dimension is denoted by $h_m$, and its output dimension is $2c_m$, where $c_m$ is the dimension of the measurement-statistic vector.
Its output is split into two $c_m$-dimensional vectors that predict the first- and second-order trends of the measurement statistics, respectively.
The energy trend head has input dimension $d_h$, hidden dimension $h_e$, and output dimension $2$, predicting the first- and second-order trends of scalar energy, respectively.
Both trend heads use layer normalization, two linear layers, and an intermediate GELU activation, and compute relative mean-squared error only for valid orders.
 
The weights are set to $(w_s,w_m,w_e)=(1.0,0.01,0.1)$. The three sub-losses are constructed using the $\operatorname{MeasVec}$ and $\operatorname{Trend}$ operators.
    
    \item \textbf{Direct-measurement parameter loss}: $\mathcal{L}_{\text{direct-self}}$ is defined in Eq.(\ref{eq:l-directself}) and calibrates the accuracy of the direct-measurement parameters for 1-step state prediction and 2-step residual prediction.
    
    \item \textbf{Self-measurement loss}: $\mathcal{L}_{\text{self}}$ is defined in Eq.(\ref{eq:l-self}) and provides an endogenous fixed-point self-supervision signal.
\end{itemize}
The loss weights are configured as follows: $\lambda_{\text{int}}=0.01$, $(\lambda_p,\lambda_f,\lambda_d)=(1,0.1,0.5)$, $\lambda_c=1$, and $\lambda_s=0.1$.

The preceding loss terms share the same per-observation record. The transition observation generated by each executed evolution step is denoted by
\begin{equation*}
\mathcal{R}_i=(\mathbf{S}_i,\mathbf{S}_i^{\text{pre}},s(\mathbf{S}_i),v_i,a_i,\mathrm{stage}_i,q_i,\nu_i,\mathcal{K}_i,
\mathbf{S}_{\text{pred},i},\Delta_{\text{self},i}^{(2)}),
\end{equation*}
where $s(\mathbf{S}_i)$ is the state summary defined in Eq.(\ref{eq:state-summary}), $v_i$ is the trend, $a_i$ is the controller action, $\mathrm{stage}_i$ is the call stage, $q_i$ is the parameter-group index, $\nu_i$ is the fast-parameter version, and $\mathcal{K}_i$ contains clipped inputs and historical observations;
the last two fields are present only for simulation-based measurements with $K=2$. This record is the common input to direct-measurement parameter calibration, slow-update preview, and PPO policy updates.

\subsection{Separate Update Rules}

The learnable parameters are divided into three families according to their time scales: slow parameters $\Slow$ are the state network's parameters shared across sequences, fast parameters $\Fast_u$ are its sequence-local parameters, and control parameters $\boldsymbol{\psi}$ are the controller-network parameters. The three families have fundamentally different time scales: $\Slow$ is updated by batch, $\Fast_u$ by events within a sequence, and $\boldsymbol{\psi}$ by policy trajectory. This heterogeneity means that the update mechanisms cannot be unified and must remain separate. If the gradient paths of the three parameter families interfere with one another, the endogenous self-supervision signal will damage general representations or allow the controller to exploit self-measurement to game the objective. This separation is a grouping at the optimization level: MetaPred's state $\mathbf{S}_{t,u}$ remains a single global tensor, and all parameters within $\Slow$ are updated synchronously in one update, corresponding to the dynamical weights on the SiFu graph without a layer hierarchy.

The three families differ in their sharing scope: $\Slow$ and $\boldsymbol{\psi}$ are shared across sequences, while $\Fast_u$ is sequence-local. Their scope and update mechanisms are shown in Table~\ref{tab:param-groups}.

\begin{table}[htp]
\caption{Assignment and update mechanisms of the three parameter families}
\label{tab:param-groups}
\centering\small
\setlength{\tabcolsep}{4pt}
\begin{tabular}{>{\raggedright\arraybackslash}p{0.15\textwidth}>{\raggedright\arraybackslash}p{0.39\textwidth}>{\raggedright\arraybackslash}p{0.21\textwidth}>{\raggedright\arraybackslash}p{0.17\textwidth}}
\toprule
Parameter family & Scope & Update source & Trigger condition \\
\midrule
Slow parameters $\Slow$ & Input projection matrix $E_{\text{in}}$, output projection matrices $W_{\text{out}}, b_{\text{out}}$; slot projection vectors $\{s_k, p_k\}$; modal reference vectors $\{\mathbf m_k\}$; all weights of the five stages of one-step evolution; readout-path weights; slow prediction parameters $p_{\Slow}$ and slow correction parameters $c_{\Slow}$ & Direct gradients of $\mathcal{L}_{\text{outer}}$ (AdamW) & Slow-parameter update boundary $s$ \\
Fast parameters $\Fast_u$ & Fast prediction parameters $p_{\Fast_u}$ and fast correction parameters $c_{\Fast_u}$ & Direct gradients of $\mathcal{L}_{\text{inner}}$ (clipped SGD) & Each controller request within a sequence \\
Control parameters $\boldsymbol{\psi}$ & All controller-network parameters, including the shared encoder, six action heads, and value head & PPO policy gradients (independent Adam optimizer) & When the policy trajectory matures \\
\bottomrule
\end{tabular}
\end{table}

The update sources of the three families in the table fall into two categories.
The first is direct gradients: $\Slow$ uses $\mathcal{L}_{\text{outer}}$, and $\Fast_u$ uses $\mathcal{L}_{\text{inner}}$, both taking gradients of fixed loss functions.
The second is policy gradients: $\boldsymbol{\psi}$ is updated by PPO, whose objective $\mathcal{L}_{\text{PPO}}$ depends on sampled trajectories and delayed advantage estimates, and therefore uses policy gradients.
These two sources determine the differences in update frequency and optimizer design.

\begin{itemize}
    \item \textbf{Slow-parameter $\Slow$ update}: the AdamW optimizer is applied at slow-parameter update boundary $s$:
    \[
    \bar{g}_s^{\Slow} = \operatorname{ClipNorm}(g_s^{\Slow}, c_{\Slow}), \quad \Slow_{s+1} = \Slow_s + a_s \cdot \text{AdamW}(\Slow_s, \bar{g}_s^{\Slow}; \alpha_s, \beta_1, \beta_2, \lambda_w)
    \]
    where $g_s^{\Slow} = \nabla_{\Slow_s} \overline{\mathcal{L}}_s$ is the accumulated batch gradient, $a_s$ is the candidate scaling factor (1 for ordinary steps and 0.25/0.5/1 for preview steps), and $c_{\Slow}$ is the global gradient-norm limit.
    
    An optional \textbf{slow-update preview mechanism} can be enabled: before the update, candidate parameters are constructed with candidate scaling factors $a\in\{1,0.5,0.25\}$, and the joint objective $J_a$ is computed by replaying the same training segment. Among the candidates satisfying $J_a \leq J_0 + \tau(J_0)$, where $\tau(z)=\epsilon_{\text{abs}}+\epsilon_{\text{rel}}|z|$ is the tolerance, the candidate with the smallest $J_a$ is selected; if no candidate qualifies, the original parameters are retained.

    Slow parameters are updated synchronously at batch boundaries because $\Slow$ covers the entire SiFu graph and optimization operates on the global parameter group; there is no layer hierarchy along which to backpropagate layer by layer.
    
    \item \textbf{Fast-parameter $\Fast_u$ update}: clipped SGD is applied when the controller enables a fast-parameter update:
    \begin{equation}
     g_u^{\Fast} = \nabla_{\Fast_u} \mathcal{L}_{\text{inner}}, \quad \Fast_u^+ = \Fast_u - \alpha_{\Fast} a_u^{\text{str}} \operatorname{ClipNorm}(g_u^{\Fast}, c_{\Fast}),
     \label{eq:fast-update}
    \end{equation}
    Let $\mathcal{M}_u$ denote the set of controller-selected observations that matured before evolution step $u$:
    \begin{equation}
     \mathcal{L}_{\text{inner}} = \mathbb{E}_{i\in\mathcal{M}_u} \left[\lambda_p\mathcal{L}_{\text{pred},i} + \lambda_d^{\text{fast}}\mathcal{L}_{\text{state},i} + \lambda_c\mathcal{L}_{\text{direct-self},i}\right],
     \label{eq:inner}
    \end{equation}
    The inner update includes state prediction, state trend, and direct-measurement losses; $a_u^{\text{str}}$ is the update strength selected by the controller. Here, $\mathcal{L}_{\text{state},i}$ is the per-observation state trend loss, constructed in the same way as $\mathcal{L}_{\text{state}}$. The inner-update weight is $\lambda_d^{\text{fast}}=0.001$.

    Fast parameters use clipped SGD because their updates are restricted to the current sequence, and clipping prevents an excessively large individual update from disrupting the state trajectory.
    
    \item \textbf{Control-parameter $\boldsymbol{\psi}$ update}: PPO is applied with a delayed update once a mature policy trajectory is ready:
    \begin{itemize}
        \item[] Reward function: $R_u = \Delta E_{\text{direct}} + \Delta E_{\text{corr}} - \beta_c C_u - \beta_w\|\Delta\Fast_u\|^2 - \beta_s\|\delta\mathbf{S}_u\|_{\operatorname{RMS}}^2 - \beta_{\infty}\mathbf{1}[\mathbf{S}_{u+1}\text{ contains non-finite values}]$, where $\Delta E_{\text{direct}}$ is the improvement in the direct-measurement parameters on independent held-out observations, $\Delta E_{\text{corr}}$ is the improvement in self-consistency under fixed conditions, and $C_u$ is the accumulated cost of measurement, updates, and correction.

        Define the direct-measurement parameter energy on a held-out observation as
\begin{equation}
E_{\text{direct}}(\Fast;r)=\lambda_p\mathcal{L}_{\text{pred},r}(\Fast)+\lambda_d^{\text{fast}}\mathcal{L}_{\text{state},r}(\Fast)+\lambda_c\mathcal{L}_{\text{direct-self},r}(\Fast),
\end{equation}
where $r$ is an independent held-out observation. Then $\Delta E_{\text{direct}} = E_{\text{direct}}(\Fast_u;r_u) - E_{\text{direct}}(\Fast_u^+;r_u)$. Define fixed-condition self-consistency on the configured probe call as
\begin{equation}
F_q(X)=\operatorname{RelMSE}\left(\Phi_q^{\text{step}}(X; \chi_u^{\text{real}})_C, X_C\right),
\end{equation}
where subscript $C$ selects the sub-tensor of the coordination region $C=\{E\}\cup M$. Then $\Delta E_{\text{corr}} = F_q(\mathbf{S}_{u+1}^{\text{pre}}) - F_q(\mathbf{S}_{u+1})$. Both evaluations retain the selected parameter-group index and do not advance the actual call counter.
        \item[]  Advantage estimation: generalized advantage estimation (GAE) is used to compute the advantage function $A_u$, with target value $G_u = A_u + \bar{V}_u$.
\item[]  PPO loss: $\mathcal{L}_{\text{PPO}} = -\mathbb{E}_u\left[\min(\rho_u\widetilde{A}_u, \bar{\rho}_u\widetilde{A}_u)\right] + c_V\mathbb{E}_u(V_{\boldsymbol{\psi}}(o_u)-G_u)^2 - c_H\mathbb{E}_u\mathcal{H}(\pi_{\boldsymbol{\psi}}(\cdot|o_u))$, where $\rho_u = \exp(\log\pi_{\boldsymbol{\psi}}(a_u|o_u) - \bar{\ell}_u)$ is the policy probability ratio, $\bar{\ell}_u$ is the behavior policy's log-probability of action $a_u$, and $\bar{\rho}_u=\operatorname{ClipTo}(\rho_u,0.8,1.2)$ is the clipped probability ratio.
    \end{itemize}
    The control parameters $\boldsymbol{\psi}$ are randomly initialized and trained from scratch, independently of the other parameters.

    The controller can only be trained through reinforcement learning methods such as PPO because its objective is the long-term return of resource allocation. This return can be evaluated only several evolution steps after an action and cannot be obtained by taking direct gradients of a single-step loss. PPO updates only $\boldsymbol{\psi}$; $\Slow$ and $\Fast_u$ are updated by their respective optimizers.
\end{itemize}

\subsection{Complete Training Procedure}
The training and inference procedures both use the standard dynamics routine, with scheduling rules implementing the differences among the three operating modes. The update timing and trigger conditions for the three parameter families are listed in Table~\ref{tab:param-groups}. Algorithm~\ref{alg:training} presents the complete training procedure for BriLLM-MetaPred. This loop body is the unified internal transition subroutine in the implementation, and encoding, autonomous, and generation calls follow the same order; the first readout in the generation stage uses the current state, and subsequent readouts occur after the previously predicted token has been written and one generation evolution step has been executed.

\begin{algorithm}[t]
\caption{BriLLM-MetaPred Training Algorithm}
\label{alg:training}
\KwIn{ Training corpus, evolution-step budgets $K_{\text{in}}, K_{\text{ev}}$, slow-parameter learning rate $\alpha_\theta$, fast-parameter learning rate $\alpha_\phi$, PPO update frequency $T_{\text{ppo}}$}
\KwOut{Trained $\Slow, \Fast_u, \boldsymbol{\psi}$}
Initialize slow parameters $\Slow$, control parameters $\boldsymbol{\psi}$, and fast parameters $\Fast_u$.

\For{each training batch}{

    \For{each training token sequence $x_{1:T}$}{
        Initialize the sequence runtime $z_0 = (\mathbf{0}, \mathbf{0}, \emptyset, (B_0^m, B_0^w), 0, \xi_0)$.
        
        \For{$t=1$ to $T$}{
            (1) Write the token projection vector $e_t$ to the input window and update the state $\mathbf{S}_{t,u}$.
            
            (2) Execute $K_{\text{in}}$ encoding evolution steps $\Phi_q^{\text{step}}$ and advance evolution step $u$.
            
            (3) Store the evolution snapshot in the history buffer $\mathcal{H}_u$.
        }
    After the context has been written, execute $K_{\text{ev}}$ autonomous evolution steps $\Phi_q^{\text{step}}$ with no new token injection, and advance $u$.
    
    \For{$j=1$ to $H$ (prediction length)}{
        (1) The controller selects action $a_u$ using the pre-action state summary and mature observations.
        
        (2) If the action selects a measurement tier, perform the corresponding measurement to obtain the residual $\widetilde{\Delta}_{\text{self}}^{(K)}(u)$.
        
        (3) If fast-parameter updates are enabled, compute $\mathcal{L}_{\text{inner}}$ using only previously mature observations, and update $\Fast_u$.
        
        (4) Issue a one-step prediction of the internal state using the updated fast parameters.
        
        (5) Execute the actual dynamics $\Phi_q^{\text{step}}$, advance $u$, obtain the pre-correction successor $\mathbf{S}_{u+1}^{\text{pre}}$, and mark the prediction record as mature.
        
        (6) If the controller requests state correction, recompute the direct-measurement residual, add the correction $\delta\mathbf{S}_u$, and update $\mathbf{S}_{u+1}$.
        
        (7) Read out the predicted distribution $p_j$ from the current dynamical state and compute the external task loss $\mathcal{L}_{\text{task}}$.
        
        (8) Store the transition observation $\mathcal{R}_i$ in the history buffer $\mathcal{H}_u$ for subsequent PPO training.
        
        (9) If $j<H$, write the ground-truth token (teacher forcing) or generated token (autoregressive generation) to the input window, and advance $t$.
    }
    Accumulate the batch gradients and apply AdamW to update the slow parameters $\Slow$.
    
    \If{the accumulated PPO trajectory length reaches $T_{\text{ppo}}$}{
        Compute the advantage function using GAE and apply PPO to update the control parameters $\boldsymbol{\psi}$.
    }
    }
}
\end{algorithm}

\section{Experiments}
\label{sec:experiments}

This section evaluates the predictive performance of the implemented BriLLM-MetaPred as a complete language model, and further examines its internal mechanisms, fast-parameter adaptation, and training at a larger scale.

\subsection{Experimental Questions and Comparison Setup}
\label{sec:exp-setup}

Table~\ref{tab:exp-questions} presents the purpose of each group of experiments. The continuous-execution training regime is used for the main language-modeling comparison; the sparse-auxiliary training regime is used to analyze internal mechanisms. The main comparison scale is 13M; in addition, a 129M model at the same scale as GPT-1 is trained to evaluate training and generation at a larger scale.

\begin{table}[htbp]
\caption{Overview of the experiments.}
\label{tab:exp-questions}
\centering\small
\begin{tabular}{@{}p{0.17\linewidth}p{0.41\linewidth}p{0.34\linewidth}@{}}
\toprule
\textbf{Research question} & \textbf{Experimental setup} & \textbf{Main observation} \\
\midrule
Language-modeling performance & 13M continuous-execution training regime versus a same-size GPT baseline & Validation PPL \\
Role of internal mechanisms & 13M mechanism ablations in the sparse-auxiliary training regime & Prediction loss in Frozen mode \\
Training at scale and generation & 129M sparse-auxiliary training regime, two window lengths & Training endpoints, validation metrics, and generation examples \\
\bottomrule
\end{tabular}
\end{table}
\paragraph{Training regimes and evaluation modes.}
This section uses the same names for MetaPred's three operating modes and its training regimes. The experiments involve two training regimes: the \textbf{continuous-execution training regime} retains node states and fast parameters across segments within a sequence, while the \textbf{sparse-auxiliary training regime} periodically runs the self-measurement mechanism alongside backbone training. Evaluation is performed in \textsc{Frozen} and \textsc{Adaptive} modes, which differ only in whether fast-parameter updates and state correction are allowed. Below, ``scale'' denotes the parameter count (M denotes million), and ``length'' denotes the training window length.

\paragraph{Evaluation metrics and reporting.}
Language modeling is evaluated using the mean negative log-likelihood (NLL) on the validation set, perplexity $\mathrm{PPL}=\exp(\mathrm{NLL})$, and bits per character $\mathrm{BPC}=\mathrm{NLL}/\ln2$ for character-level tasks; lower values are better for all three metrics. Three-seed experiments use 42, 43, and 44 and report the mean and sample standard deviation across seeds. Teacher forcing first scores the prediction at the current position and then writes the reference token; fast-parameter updates in \textsc{Adaptive} mode use internal feedback.

\paragraph{Chinese data split.}
The Chinese experiments use the THUCNews-v3 dataset, split at a 90:5:5 ratio by groups of exact duplicates and conservatively identified near-duplicate articles. The training, validation, and test sets contain 728,267, 40,382, and 40,533 articles, respectively; the vocabulary contains 6,512 tokens that occur at least 5 times in the training set. Full hyperparameters and optimizer configurations are provided in Appendix~\ref{sec:implementation}.

\paragraph{English data split.}
The English experiments use BookCorpusOpen. Exact whole-book deduplication with SHA-256 retains 17,785 books from the 17,868 original documents. A split based on document-content hashes at a 98:1:1 ratio yields 17,440, 167, and 178 books in the training, validation, and test sets, respectively. After further excluding books in which the proportion of unknown BPE fragments exceeds 1\%, the three sets retain 17,332, 167, and 176 books, respectively. Both models share the official GPT-1 BPE vocabulary of 40,478 entries. Each book is independently divided into non-overlapping 512-token blocks, without concatenation across books, and any final incomplete block is discarded; the first 16 tokens of each block provide context, and the remaining 496 tokens are used for evaluation.

\paragraph{Baseline model.}
The comparison baseline, referred to as the GPT baseline, is implemented strictly according to the GPT-1 specification and scaled to the corresponding size as needed for comparison. The 13M baseline is the same size as BriLLM-MetaPred (approximately 13.2M parameters) and is used for comparisons under a controlled training budget; the 129M baseline follows the official GPT-1 parameter scale. The baseline model and BriLLM-MetaPred use the same data splits and vocabulary and differ in model architecture and training objective. All references to the ``GPT baseline'' below denote this type of comparison model.

\subsection{Language-Modeling Performance of the Complete Model}
\label{sec:exp-language-modeling}

We first examine whether BriLLM-MetaPred in the continuous-execution training regime can achieve competitive language-modeling performance under a controlled training budget. The comparison uses a GPT baseline implemented according to the GPT-1 specification at the same scale as BriLLM-MetaPred (approximately 13.2M parameters); both models share the same data splits and GPT-1 BPE vocabulary. BriLLM-MetaPred and the baseline have approximately 13.22M and 13.21M parameters, respectively, and are each trained for 788 steps using 25,014,272 training tokens. Each 512-token sequence block uses the first 16 tokens as context and the remaining 496 for evaluation; the model retains local runtime information throughout each block.

Table~\ref{tab:exp-bookcorpus}~compares two seeds for which evaluation has been completed at the training endpoint. Both cover the full validation set of 25,540 blocks and 12,667,840 training tokens; BriLLM-MetaPred uses \textsc{Adaptive} mode. The relative PPL reduction is calculated against the GPT baseline with the same seed.

\begin{table}[htbp]
\caption{Model performance comparison on the BookCorpus validation set under equal training budgets.}
\label{tab:exp-bookcorpus}
\centering\small\setlength{\tabcolsep}{3pt}\renewcommand{\arraystretch}{1.22}
\begin{tabular}{@{}>{\raggedright\arraybackslash}p{\dimexpr0.16\linewidth-6pt\relax}>{\centering\arraybackslash}p{\dimexpr0.26\linewidth-6pt\relax}>{\centering\arraybackslash}p{\dimexpr0.29\linewidth-6pt\relax}>{\centering\arraybackslash}p{\dimexpr0.25\linewidth-6pt\relax}@{}}
\toprule
\textbf{Seed} & \textbf{GPT baseline PPL $\downarrow$} & \textbf{BriLLM-MetaPred PPL $\downarrow$} & \textbf{Relative PPL reduction} \\
\midrule
43 & 278.847 & $\bm{229.451}$ & $\bm{17.71\%}$ \\
44 & 279.355 & $\bm{248.554}$ & $\bm{11.03\%}$ \\
\bottomrule
\end{tabular}
\end{table}

PPL decreases by approximately 11\%–18\% across the two seeds, indicating that the continuous-execution training regime achieves good language-modeling performance under matched training budgets. This comparison includes the combined effects of model structure, internal auxiliary objectives, and execution. The two models use different training backends, and their actual computational costs have not been matched; the table supports a predictive comparison under the conditions described above.

\subsection{Role of Internal Mechanisms}
\label{sec:exp-components}

We keep the training configuration fixed and vary only three groups of core mechanisms to examine the effects of modal nodes, training-inference coupling, and the fixed-point objective on predictive performance. All variants are trained independently; evaluation is performed in \textsc{Frozen} mode so that fast-parameter updates or state correction during evaluation do not obscure the differences in training mechanisms.

\begin{table}[htbp]
\caption{Core-mechanism ablations: mean and sample standard deviation over three seeds; $\Delta\mathrm{PPL}$ is defined as the full-model PPL minus the variant PPL.}
\label{tab:exp-core-ablations}
\centering\small
\begin{tabular}{@{}lcc@{}}
\toprule
\textbf{Training variant} & \textbf{Validation PPL $\downarrow$} & \textbf{$\Delta$PPL} \\
\midrule
Full model & $36.700\pm0.832$ & -- \\
Modal nodes disconnected & $37.887\pm0.734$ & $-1.187$ \\
Fast-parameter updates and state correction disabled & $36.657\pm0.599$ & $+0.043$ \\
Fixed-point objective removed & $36.678\pm0.906$ & $+0.022$ \\
\bottomrule
\end{tabular}
\end{table}

The full model retains all mechanisms. The other three ablation variants are configured as follows:
``Modal nodes disconnected'' zeros all signal paths between the modal node region $M$ and the other node regions, so that $M$ neither receives nor sends signals. At this point, the coordination region $C$ contains only the self-coordination node $E$, which independently performs all global coordination functions, while the working node region $W$ continues to evolve independently on the same state tensor;
``Fast-parameter updates and state correction disabled'' retains measurement and the shared internal objectives but disables state correction, fast-parameter updates, and PPO data collection; ``Fixed-point objective removed'' only sets the corresponding training coefficient to zero. The three ablation variants are trained independently and are all evaluated in \textsc{Frozen} mode.

The results in Table~\ref{tab:exp-core-ablations}~show that disconnecting the modal nodes increases validation PPL by an average of 1.187, with the same direction of change across all three seeds, indicating a measurable contribution of the modal node region $M$ to predictive performance. The limited increase also shows that the model retains its overall predictive capability after the modal nodes are disconnected, which in turn indicates that the self-coordination node $E$ can still independently support the model's global coordination function without $M$. Disabling fast-parameter updates and state correction or removing the fixed-point objective has a small effect on PPL in \textsc{Frozen} mode. This indicates that the additional constraints they provide during training have a weak relationship with predictive performance under \textsc{Frozen} evaluation, and that their effects are more concentrated on adaptation during execution.

\subsection{Training-Inference Coupling}
\label{sec:exp-online-cost}

This subsection examines the actual engineering execution of training-inference coupling: a trained checkpoint is held fixed, and only the allowance for fast-parameter updates and state correction during evaluation is switched. Table~\ref{tab:app-exp-smoke} presents paired runtime diagnostics for an intermediate BookCorpus checkpoint.

\begin{table}[htbp]
\caption{Runtime diagnostics for an intermediate BookCorpus checkpoint. Seed 42, step 100, 16 validation blocks, and 7,936 evaluation tokens; GPU memory is the approximate peak allocated amount.}
\label{tab:app-exp-smoke}
\centering\small\setlength{\tabcolsep}{3pt}\renewcommand{\arraystretch}{1.22}
\begin{tabular}{@{}>{\raggedright\arraybackslash}p{\dimexpr0.54\linewidth-6pt\relax}>{\centering\arraybackslash}p{\dimexpr0.23\linewidth-6pt\relax}>{\centering\arraybackslash}p{\dimexpr0.23\linewidth-6pt\relax}@{}}
\toprule
\textbf{Metric} & \textbf{\textsc{Frozen}} & \textbf{\textsc{Adaptive}} \\
\midrule
NLL $\downarrow$ & 8.512188 & 8.512103 \\
Actual state evolution count & 8,256 & 8,256 \\
Simulation call count & 3,120 & 3,120 \\
Simulation-based measurement count & 528 & 528 \\
Direct measurement count & 64 & 64 \\
Fast-parameter update count & 0 & $\bm{1,056}$ \\
State correction count & 0 & $\bm{528}$ \\
Peak allocated GPU memory (MB) & $\approx 112$ & $\bm{\approx 175}$ \\
\bottomrule
\end{tabular}
\end{table}

\textsc{Adaptive} adds 1,056 fast-parameter updates and 528 state corrections, with higher recorded GPU memory usage. The mean relative error of this checkpoint's direct-measurement residual is 0.653, above the calibration threshold of 0.5, and 64 corrections driven by unreliable direct measurements are suppressed. \textsc{Frozen} also performs diagnostic measurement; the cost difference between the two columns reflects only the additional adaptation and correction and cannot be used to estimate the overhead of the entire measurement module.

\subsection{Training Stability and Scaling}
\label{sec:exp-scale-generation}

Figure~\ref{fig:exp-training-loss-en}~shows the complete 788-step training NLL curves for both models at seed 42. The horizontal axis uses the cumulative number of training tokens so that the training budget is consistent with Table~\ref{tab:exp-bookcorpus}. The comparison of the training curves shows a smoother pattern of loss reduction for BriLLM-MetaPred, which also indicates more stable training dynamics.

\begin{figure}[htbp]
\centering\includegraphics[width=0.94\linewidth]{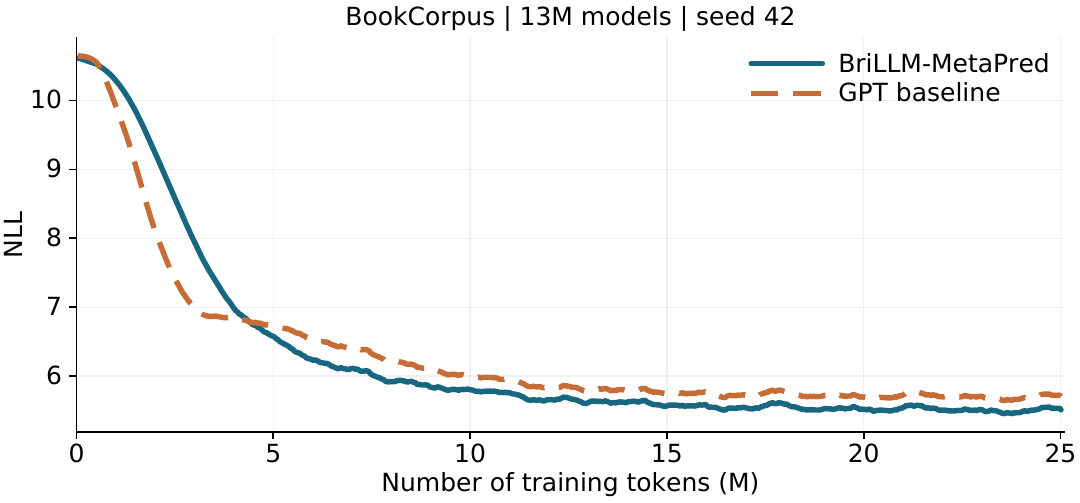}
\caption{Training NLL of 13M BriLLM-MetaPred and the GPT baseline on BookCorpus (seed 42).}
\label{fig:exp-training-loss-en}
\end{figure}

We further examine whether a model with approximately 129M parameters can complete training and observe its generation behavior. The model uses the sparse-auxiliary training regime on THUCNews, with one run at each of two training window lengths. The two runs have different training budgets and evaluation coverage, so no direct comparison is made between the rows. The final training results are presented in Table~\ref{tab:exp-scale}.

\begin{table}[htbp]
\caption{129M BriLLM-MetaPred completes training and evaluation at both training window lengths, each with seed 42.}
\label{tab:exp-scale}
\centering\small\setlength{\tabcolsep}{3pt}\renewcommand{\arraystretch}{1.22}
\begin{tabular}{@{}>{\raggedright\arraybackslash}p{\dimexpr0.26\linewidth-6pt\relax}>{\centering\arraybackslash}p{\dimexpr0.20\linewidth-6pt\relax}>{\centering\arraybackslash}p{\dimexpr0.21\linewidth-6pt\relax}>{\centering\arraybackslash}p{\dimexpr0.15\linewidth-6pt\relax}>{\centering\arraybackslash}p{\dimexpr0.18\linewidth-6pt\relax}@{}}
\toprule
\textbf{Input$\rightarrow$output (tokens)} & \textbf{Training set (tokens)} & \textbf{Validation set (tokens)} & \textbf{BPC $\downarrow$} & \textbf{PPL $\downarrow$} \\
\midrule
$16\rightarrow16$ & 470,404,384 & 16,384 & 3.883050 & $\bm{14.754158}$ \\
$512\rightarrow512$ & 133,826,560 & 7,635,968 & 4.766332 & $\bm{27.215033}$ \\
\bottomrule
\end{tabular}
\end{table}

The short-window run completes 459,380 updates, while the long-sequence run completes 4,085 updates and is evaluated on 14,914 validation blocks. The two rows differ in training budget, evaluation coverage, and update count, so the metrics are not compared directly across rows; the table shows that training and evaluation can be completed at both training window lengths at a scale of approximately 129M parameters. Because the number of training targets and evaluation coverage of the GPT baseline differ from those of our runs, the current results do not establish a scale-performance ranking or validate a fitted scaling law.

\subsection{Language Generation Examples}
\label{sec:exp-generation}

\paragraph{Chinese generation comparison with the same inputs.}
Table~\ref{tab:exp-generation} presents the outputs of 13M BriLLM-MetaPred and the corresponding GPT baseline model. Both models use training seed 42, receive the same complete prompts, and generate 50-token continuations by greedy decoding. Each prompt is independently reset, and no teacher forcing is used.

\begingroup\small\setlength{\tabcolsep}{3pt}\renewcommand{\arraystretch}{1.12}
\begin{longtable}{@{}>{\raggedright\arraybackslash}p{0.20\linewidth}>{\raggedright\arraybackslash}p{0.37\linewidth}>{\raggedright\arraybackslash}p{0.37\linewidth}@{}}
\caption{Chinese generation comparison with the same inputs.}\label{tab:exp-generation}\\
\toprule
\textbf{Input prompt} & \textbf{BriLLM-MetaPred} & \textbf{GPT baseline} \\
\midrule\endfirsthead
\multicolumn{3}{l}{\small Table~\thetable~(continued)}\\\toprule
\textbf{Input prompt} & \textbf{BriLLM-MetaPred} & \textbf{GPT baseline} \\
\midrule\endhead
\bottomrule\endfoot
01 中国队在第二场比赛中3比0大胜 & 。此外，中国队还有1000万人次，中国队在2008年10月份以来，中国队在2008年10月份以来首次 & 日本队，中国队在第三场比赛中以1比0战胜日本队，中国队在第四场比赛中以1比0战胜日本队，中国队在第四 \\ \midrule
02 欧盟决定为经济疲弱的成员国提供总额达 & 到1.2亿美元。此外，美国政府将继续保持稳定增长。此外，美国政府将在2010年底前将有望进一步加强国 & 到1000亿欧元的纾困金，这是欧盟决定为经济增长提供的最大障碍。欧盟委员会主席巴罗佐表示，欧盟委员会 \\ \midrule
03 最近中医膏方养身保健成为了 & 保障房的安全保障，保障房建设保障性安居工程建设安全保障工程建设安全保障安全保障措施保障安全保障安全保 & 一个新的消费方式，但是，这种方式的确是一种新的消费方式，但是，这种方式的确是一种新的消费方式，但是， \\ \midrule
04 消费者在购买家具的时候非常注重 & 要的原因是，这些都是在网上看到的。”“在这个时候，我们看到了这个消息，我们看到了这个时候，我们看到了 & 家具的质量，而且家具的质量也是非常重要的。”在家具行业，家具的质量是否符合国家标准的标准，家具的质量 \\ \midrule
05 儿童房甲醛超标5倍，即有引发 & 网络购物网络购物网站，网络购物网站、网络购物网站、网络网络建设、网络网络建设网络建设网络建设网络建设 & 甲醛的可能。”“甲醛超标”是一种新型甲醛释放量超标的新型甲醛释放量超标，这种新型甲醛释放量超标的新型 \\ \midrule
06 双色球一等奖的奖金有望达到 & 1000万美元。而在这个过程中，我们认为，这样的结构性机会是一个非常有效的解决方案。“我们的目标是在 & 1000万元。据悉，这是中国彩票双色球第10009期开奖的一张彩票，中奖号码为“0311111111 \\ \midrule
\end{longtable}
\endgroup

Both models produce locally coherent news-style text after some prompts, but both exhibit repetition, topic drift, and unsupported numbers.

\paragraph{English generation comparison with the same inputs.}
Table~\ref{tab:exp-generation-en}~presents same-prompt continuations from two approximately 13M models on BookCorpus. Both use training seed 42 and the checkpoint at step 788; with the same sampling seed, the prompt consists of the first 16 BPE tokens of a fixed validation block. Each continuation contains 128 tokens, with temperature 0.8 and top-k 40; each prompt is independently reset. BriLLM-MetaPred uses \textsc{Adaptive} mode, and the generation process does not read subsequent reference tokens.
The main text presents the first six examples in a fixed order, retaining each full continuation; line breaks in the outputs are marked by ``[/R]''.

Both models can generate English phrases and locally coherent text, but also exhibit repeated punctuation, grammatical mismatches, and long-range semantic incoherence.

\begingroup\small\setlength{\tabcolsep}{3pt}\renewcommand{\arraystretch}{1.12}
\begin{longtable}{@{}>{\raggedright\arraybackslash}p{0.48\linewidth}>{\raggedright\arraybackslash}p{0.48\linewidth}@{}}
\caption{English generation comparison with the same inputs.}\label{tab:exp-generation-en}\\
\toprule
\textbf{BriLLM-MetaPred} & \textbf{GPT baseline} \\
\midrule\endfirsthead
\multicolumn{2}{l}{\small Table~\thetable~(continued)}\\\toprule
\textbf{BriLLM-MetaPred} & \textbf{GPT baseline} \\
\midrule\endhead
\bottomrule\endfoot
\multicolumn{2}{@{}p{0.98\linewidth}@{}}{\textbf{01 Common input:} gregory zorzos  [/R]  100 alphanumeric crosswords  [/R]  gregory zor}\\* \hline
{\_ - one and the same one of your body than to see the the room - at a day of his feet ? "  [/R]  " , " no , the house and the house of the side . and the day of the world .  [/R]  "  [/R]  one of his life , we were in a different day , and the way you 're sure , " if the way had out for a long time to do , we do n't think it . i have the door and did n't be back at the world . but it was some important to be more . and this we 'm not a bit of the rest , and the room was in the small } & {- the ,  [/R]  he said \_  [/R]  , as but it .  [/R]  \_ " yes , and the one .  [/R]  chapter 4  [/R]  the ground , a good - 2 :  [/R]  the floor of the other the air .  [/R]  it , and the day that '  [/R]  '  [/R]  ' : * *  [/R]  the last day , this 's an year is the house . he is his first and he was not not like the most of the door and the same . if i have .  [/R]  " i 'll have no good \_ , the air he was a bit of them from the most of the first , and i have to be going to get a lot } \\ \midrule
\multicolumn{2}{@{}p{0.98\linewidth}@{}}{\textbf{02 Common input:} lover 's knot . "  [/R]  " i 'm serious , morris . i want a }\\* \hline
{second world , i do n't know that it 's been a new , " \_ \_ \_ \_ he would see a good man and do what it is to be a moment . he was a few of the floor of the other , and how happened to you do n't be . " you know , she 's not good .  [/R]  i need the last place i know who " , " i 'm a idea what she asked . "  [/R]  " what was there or this he was a few hundred than i have just to understand her . "  [/R]  " she said , " if you have it . \_ "  [/R]  " she is a better more than a } & {right . " i do n't see you , " i do n't have to it was too . . " she would have to a a time to do you and i know ... you want to be a little . " i 'm n't know ... "  [/R]  "  [/R]  "  [/R]  "  [/R]  you have to know you . "  [/R]  " i 're n't go the same , her hair and i 'm n't want to take a more in a good . you had be one than about you . " why , " she wanted to talk with a place . " she was a little to him , he was still like a lot to go to i said . i } \\ \midrule
\multicolumn{2}{@{}p{0.98\linewidth}@{}}{\textbf{03 Common input:} ibn sa'ûd assembled muslims and , slandering them , said , }\\* \hline
{" i have come to the one . the long thing .  [/R]  " i are a little part it is on to be an . the new .  [/R]  we 're all the day . " what he 's not so , you 're going to be going to give a time if you do . "  [/R]  " we 've been for what was to be .  [/R]  " i did n't know that , he had seen that all the ground . "  [/R]  " i had my own , " 's sure she 'd never been a good man .  [/R]  " even that , to have not much to get on my father . "  [/R]  * " he said . "  [/R]  } & {and the two a great , they are n't keep , that is a most of the other ,  [/R]  and for your way that was to the way to you have to the same , that 's just a very more than you are you to be the way to have the table of the rest of the people is .  [/R]  but this and they were the house , and a the woman , the first of the little day , who in a whole .  [/R]  that had been the other ; . you were and the other of the table .  [/R]  no , when it , in the the next , he was a life , , the time , and her for } \\ \midrule
\multicolumn{2}{@{}p{0.98\linewidth}@{}}{\textbf{04 Common input:} sured caressing . " ah . i see , " he said slowly , a trace }\\* \hline
{of the man was n't not here at the house .  [/R]  " i go to the two years . i 'm already just a one of their side . she did n't need to be in the two , i have done to the table and the man would n't look . "  [/R]  " i 'd thought i could n't make what i was to be able to see the two thing . the good day .  [/R]  they 'd found your eyes . " the other face is a little the night and he told us for him . but "  [/R]  " we had to the one , the other . he looked into the air . " he said . " i 'm } & {, "  [/R]  " i 'd made to you did n't know . "  [/R]  "  [/R]  " he was not to stop you . "  [/R]  "  [/R]  " she could n't .  [/R]  " i could n't you you know she 's all . "  [/R]  " he had her , "  [/R]  "  [/R]  " let her , "  [/R]  " i did i 'm the first ? "  [/R]  "  [/R]  " she 'd be . they was going to do n't be like a good . " i want to the ground . they do not in this . "  [/R]  " i 'm that they did n't know . " do you do n't know i have to think you have to } \\ \midrule
\multicolumn{2}{@{}p{0.98\linewidth}@{}}{\textbf{05 Common input:} - looking mahogany bar off to the side . " let 's get a drink . }\\* \hline
{ [/R]  " i mean .  [/R]  " i do n't know . "  [/R]  " you have to give me to do that . '  [/R]  i 'll be on you . "  [/R]  " she said ; " it , " said that , not . " what you do n't do you . "  [/R]  " what will i have to be that for you at the way , he , " 's right . "  [/R]  " let it you be , " you want to make a bit . " she 's all n't i 'll have to be it and you , just . "  [/R]  " you are you not a place and you . " we can just like that , } & {" i n't believe it was you . "  [/R]  " i 'll n't get to get his life . " it ? " you 're n't want to her , "  [/R]  "  [/R]  "  [/R]  when you to look . " she 's face . " what was n't do n't have to be going to you were a few of the world 's hand of the way to you could n't go into the last a more i know he took her face with it is it 's just just . she had been out of me , my head the first . "  [/R]  " she 's n't have you do anything and i could n't like you know it will hear her , } \\ \midrule
\end{longtable}
\endgroup

In terms of context sensitivity, BriLLM-MetaPred and the GPT baseline have comparable output performance in English. In Chinese, however, the GPT baseline shows stronger context awareness (i.e., the model's memory capability), while BriLLM-MetaPred performs substantially worse in this respect. This can be understood as a characteristic of its capability for creative, free-form language generation. The characteristics of both models' language-generation capabilities will be explored further in future work.

Overall, the experimental results support the goal of this paper in four respects: (1) under matched training budgets, BriLLM-MetaPred in the continuous-execution training regime achieves PPL comparable to the equivalent GPT baseline across two seeds;
(2) in the core ablations, disconnecting the modal nodes consistently degrades prediction in \textsc{Frozen} mode, indicating that the modal node region $M$ makes a measurable contribution to predictive performance; the limited increase shows that the model retains its overall predictive capability after the modal nodes are disconnected, which in turn indicates that the self-coordination node $E$ can still independently support the model's global coordination function without $M$;
(3) runtime diagnostics show that training-inference coupling does occur, and that fast-parameter updates and state correction both execute normally and are suppressed by calibration error; (4) at a scale of approximately 129M parameters, the model completes training and evaluation and generates locally coherent text. Together, these four results support the three core features listed in Section~\ref{sec:paradigm} and provide evidence for the engineering feasibility of the first language-model implementation of MetaPred.

\section{From the Primitive Specified by Mathematical Theory to the Engineering Architecture}

The mathematical theory of AGI specifies a primitive, rather than an architectural blueprint. It states that the model must have an endogenous fixed-point objective, representation and prediction must arise from the same dynamics, training and inference must proceed together, and a self-measurement mechanism must exist. It does not state that the model carrier must be a SiFu graph, that there must be two networks, or that there must be a particular number of operating modes. What, then, determines each structural choice in BriLLM-MetaPred: the theory, current hardware, or the designer's preference? Below, we assign each structural choice to one of three categories: strict theoretical requirements, uniquely feasible forms under current constraints, and surface-level parameters that do not determine paradigm identity. This classification does more than catalog the choices. It leads to a clear conclusion: \emph{on existing silicon accelerators, each key structure in BriLLM-MetaPred has only one feasible form}.

\paragraph{Strict theoretical requirements}

Four features of BriLLM-MetaPred are required by the primitive itself. Every MetaPred implementation must have corresponding features, regardless of its model carrier, scale, or hardware.

First, a single dynamical state tensor. If representation and prediction must arise from the same dynamics, they cannot be separated by an encoder--decoder decomposition; a single tensor $\mathbf{S}_{t,u}$ must support both. If representation were computed by one set of parameters and prediction by another, they would arise from two different dynamics, immediately breaking representation-prediction unity. $\mathbf{S}_{t,u}$ directly expresses this requirement.

Second, an endogenous fixed-point loss. The theory requires a signal that measures the self-consistency of the state under its own evolution. External labels cannot provide this signal: they constrain only the emitted distribution, not the state trajectory. $\mathcal{L}_{\text{fp}}$ and the self-measurement residual $\Delta_{\text{self}}^{(K)}$ are not design choices; they directly express the fixed-point requirement. Removing them would remove the self-supervised core required by the theory.

Third, the self-measurement subprocess and the main process share the same dynamics routine. The theory requires the state to measure itself. If the routine used for measurement differed from the routine used for evolution, measurement would compare the state with a different dynamics. The resulting difference would measure a discrepancy between two dynamics, rather than inconsistency within the state itself, and would not serve the theoretical requirement. Therefore, $\Phi_q^{\text{step}}$ must serve both processes, with only the counter $u$ distinguishing them. This is required to preserve the validity of the measurement semantics, rather than to save parameters.

Fourth, training and inference must share the same evolution structure. Their required coupling means that the training path must also exist during inference, and the inference path must also exist during training. Separate training and inference forward paths would break this coupling, necessarily reducing self-measurement to a training-only auxiliary loss and removing the state-correction mechanism from the inference path. It follows that any division of execution according to the permitted scope of parameter updates or state correction must define modes of the same dynamics. The theory does not require exactly three modes; it requires every mode to instantiate the same structure. The three-mode design follows from two joint requirements: evaluation without updates (\textsc{Frozen}) and deployment with bounded updates (\textsc{Adaptive}). Both are operating conditions that the theory's own premises require the model to support.

These are the theory's four hard constraints. Every MetaPred implementation must satisfy them; BriLLM-MetaPred does so.

\paragraph{Unique feasibility under current constraints}

Four additional features are required not directly by the primitive itself, but by the requirement that the primitive execute on existing hardware at an acceptable cost. Given current hardware and the four theoretical constraints, each has only one feasible form. A different hardware substrate could permit different forms; the result would be another implementation of the same MetaPred primitive, rather than a refutation of the current implementation.

First, a SiFu model carrier rather than any hierarchical graph. The theory requires a model carrier more general than the specific architectures of current machine learning, while existing accelerators execute only dense tensor operations efficiently. A hierarchical graph must maintain alignment between topology and tensor shapes, violating the theoretical requirement for weak coupling between the model carrier and dynamics. SiFu's fully connected, bidirectional graph is the only structure that satisfies both requirements: it provides maximal topological generality, and its three node-axis mixing operators map directly to dense tensor operations. Other graph structures either lose generality or cannot execute efficiently. SiFu is therefore the only feasible model carrier on current hardware, rather than a matter of designer preference.

Second, two networks rather than one or three. If the controller and state network shared an encoder, the controller's policy gradients would flow directly into the general representation. The theory's parameter-heterogeneity requirement explicitly requires the update source for control parameters (PPO) to be completely isolated from the update sources for slow and fast parameters. Adding a third network to separate the controller's policy and value heads would instead violate the parameter-sharing principle of the actor-critic architecture: the policy and value must share the same state representation, or advantage estimation becomes unreliable. Two networks are therefore the only possible decomposition: one supports state dynamics and the other the policy-value mapping, with independent parameters and separate gradient paths.

Third, five gated-mixing stages within each evolution step. The theory requires the state to be updated by the same dynamics, but does not specify how that update should be organized. Any self-updating graph dynamics must perform four functions: incorporate external information into the current state; route information according to each node's role; compress the global context into a coordination anchor; and write the coordinated result back locally. Incorporating external information requires input mixing; the last two functions require a coordination node and a broadcast operation; writeback itself requires gated writeback. Within information routing, nodes must also determine their relative emphasis on the working, modal, fast, and slow functions, which requires role assignment. The five stages are therefore not an interchangeable arrangement, but the minimum necessary decomposition of these four functions on the node graph. The stages can be subdivided further, but their number cannot be reduced: removing any stage makes it impossible to satisfy all the functional requirements simultaneously.

Fourth, a bounded auxiliary pool and a system of measurement tiers. The theory requires finite budgets for measurement and updates; otherwise, the controller has no notion of cost and cannot learn a resource-allocation policy. This requires a set of discrete tiers and a bounded pool from which the controller can sample. The tier granularity must distinguish three functions: an immediate approximation (direct measurement), shallow simulation (short-horizon), and a complete parameter-group cycle (long-horizon). Skipping measurement provides their lower bound. Four tiers are the smallest set that covers these three functions together with a zero signal. Other schemes of comparable granularity are conceivable, but they must cover at least the same four functions; otherwise, they lose either budget control or measurement depth. In the current implementation, these four tiers are the only feasible set.

These four features are uniquely feasible given the theory and current hardware. BriLLM-MetaPred adopts them because of these constraints, rather than convenience.

\paragraph{Surface-level parameters independent of paradigm identity}

The remaining choices are required neither by the theory nor by current hardware. Their values do not determine whether the model is MetaPred. Changing them changes neither a strict theoretical requirement nor a uniquely feasible form; they concern only scale, hyperparameters, and function classes.

First, the specific node count $N$, feature width $d$, and partition sizes $|I|,|O|,|M|,|W|$. Any values that preserve the reachability required by the node-axis mixing operators and keep the parameter count tractable define another instance of the same MetaPred primitive.

Second, the parameter-group count $Q$, temperature $\tau$, specific four-way role partition, number and values of correction strengths, and loss weights $\lambda_{\text{int}},\lambda_p,\lambda_f,\lambda_d,\lambda_c,\lambda_s$. These are all hyperparameters. They do not alter the four theoretical constraints or the four uniquely feasible forms described above. Their current values reflect tuning rather than derivation.

Third, the specific neural operators---GELU, layer normalization, and the particular two-layer mapping used by the state predictor---and the initialization scheme. The theory constrains the functions these operators must perform, such as providing a smooth nonlinearity through GELU or removing channel mean and scale through layer normalization, rather than the function class used to implement them. Substituting other operators that serve the same functions leaves the model a MetaPred implementation.

The value of this three-way classification goes beyond cataloging the choices. It establishes three facts. First, the theory imposes four strict requirements that no MetaPred implementation may violate. Second, current hardware imposes four uniquely feasible forms that cannot be replaced on existing accelerators. Third, designers retain freedom over three groups of surface-level parameters whose values affect performance rather than paradigm identity.

This leads to a conclusion more important than the classification itself. \emph{The paradigm identity of BriLLM-MetaPred is determined entirely by its four strict theoretical requirements and four uniquely feasible forms, independently of all surface-level parameters.} A future MetaPred implementation on different hardware, or with different tuning, may change any surface-level parameter and even any uniquely feasible engineering form. As long as it satisfies the four strict requirements, it remains a MetaPred implementation and automatically satisfies the theoretical requirements for representation-prediction unity, endogenous self-supervision, coupled training and inference, and a self-referential fixed point. This defines the connection between the mathematical theory of AGI and a specific machine: there is one primitive with potentially many implementations, while the current BriLLM-MetaPred is the \emph{only} feasible implementation that satisfies all the theoretical constraints on existing silicon accelerators.

This conclusion also provides a clear assessment of the current focus on scaling. If the preceding deductions hold, general intelligence cannot emerge from scale: increasing scale only increases the values of surface-level parameters, without changing the four strict requirements or the four uniquely feasible forms. If scaling did produce general intelligence, it would be a fortunate engineering outcome: somewhere, it would necessarily have incidentally implemented the mechanisms required by the four strict requirements, and scaling laws themselves do not guarantee that this will occur. General intelligence can only be supported by a specific model required by mathematics; BriLLM-MetaPred is the first instance of this class.

\section{Consistency Between MetaPred and Current Directions in AI Research}

The derivation of an optimal machine-learning model in the mathematical theory of AGI starts from a simple, intuitive premise: the model must support unrestricted self-reflection, that is, reasoning about its own reasoning, allowing infinite recursion and convergence \citep{zhao2026agitheory}. The result derived under this premise is MetaPred. MetaPred is therefore a mathematically self-referential machine-learning paradigm. Handling self-reference in computer science generally requires addressing the challenge of G\"odel's incompleteness theorems; Lambek's fixed-point lemma is known to support the construction of valid, G\"odel-compatible categorical self-referential structures \citep{Lambek1968}. MetaPred requires its model implementations to include a fixed-point loss term as a necessary expression of this compatibility. Our BriLLM-MetaPred satisfies this requirement.

This theoretical premise leads to three mechanistic conclusions for our model. First, even without external labels or signals, the fixed-point prediction objective provides a direction of evolution toward internal consistency, allowing the model to automatically achieve the greatest possible predictive generalization through that consistency. Second, the self-measurement mechanism performs real-time simulation within the context of dynamical evolution, giving both training and inference greater flexibility to improve efficiency and performance. Third, a single dynamical state makes training and inference proceed together: they share the same evolution structure, without separate backward paths for training and forward paths for inference. The following discussion builds on these three conclusions.

MetaPred is the first self-referential machine-learning paradigm: the same dynamical state is both the object and the carrier of evolution. The model can directly apply an endogenous adjustment to its direction of evolution on its own state, identify the current direction as unsustainable, and initiate a new direction on that same state.
Self-reference in the engineering implementation of MetaPred appears at three levels. At the level of the static model carrier, the self-coordination node $E$ represents the model's own global state. At the dynamical level, self-measurement takes the current model state as input and directly evaluates the model's own consistency. At the loss level, the fixed-point term requires state evolution to reach a self-predictive steady state in one step. All three are implemented through the same dynamics, enabling the model to evaluate and correct its own direction of evolution internally.
This provides a direct mechanism for creative thinking: a new direction begins by rejecting the old one, and both rejection and recreation are carried out by the same state through the same dynamics.

To clarify, the BriLLM-MetaPred model presented here is not a complete AGI system. According to the mathematical theory of AGI, a complete minimal AGI system contains at least four modules, each implemented as a MetaPred model: (1) a global process-control module; (2) a super-modal module, namely the language model---natural language is the super-modality that the theory predicts must be built in; (3) an input/output encoder--decoder module for perception and action; and (4) a multimodal coordination module. Through the maximal self-symmetry of its machine-learning elements, the MetaPred algorithmic framework ensures that nesting multiple MetaPred models automatically produces another MetaPred model with a fixed-point loss objective, demonstrating intrinsic self-similarity. This paper implements only the second module. The theory further requires the super-modality to be endogenous to the model, rather than placed on an equal footing with the other modalities of perception--action interaction at the model-interaction level. This implies that an AGI system must place the language modality above other multimodal information in multimodal processing. The theory also gives two interdependent necessary premises: (1) an AGI system must contain a language-model module; and (2) an AGI system must realize a perception--action loop that is fully self-consistent under unlimited self-reflection. These imply two complementary claims: an AGI system cannot be implemented entirely outside a language model, and a language model alone cannot achieve AGI. These two claims automatically free a complete AGI system from hallucinations because the perception--action loop at the interaction level provides testable external constraints and verification for purely symbolic thinking. This conclusion concerns the complete AGI system, rather than the language-model module implemented in this paper.

Given this positioning, the mechanistic significance of MetaPred can be developed at two levels.

First, recursive self-improvement. MetaPred implements representation-prediction unity through unified dynamics, with its dynamical state simultaneously playing the roles of model architecture and model state in traditional machine learning. On this basis, self-measurement is equivalent to real-time optimization of the model's own machine-learning execution at every time step, achieving optimality at the paradigm level. If the mathematical theory of AGI is correct, this realizes the vision of recursive self-improvement pursued in current AI research at the internal level of machine learning.

Second, continual learning. MetaPred's symmetry is reflected in training-inference coupling: \textsc{Train}, \textsc{Frozen}, and \textsc{Adaptive} are essentially instances of the same class of dynamical evolution. Even in \textsc{Frozen} mode, which performs no parameter updates, the model still executes self-measurement in full. This property is not an engineering addition but a mechanism required by the theory: every operating mode contains the same evolution process as training. The model's built-in mechanisms therefore realize the expectations of academia and industry for continual learning in large models.

In the limit, the fixed-point prediction objective acts as the final arbiter of the model's overall dynamics. When external labels are absent, MetaPred's dynamics remain governed by the fixed-point objective, which cannot be eliminated. This brings the model's reasoning to a state of final self-evaluation. We do not claim here that the MetaPred mechanism has achieved self-awareness or free will, but this purely mathematical reasoning provides a mathematically and computationally optimal path toward such mechanisms.

Finally, consider the derivation based on the natural-balance principle in the mathematical theory of AGI \citep{zhao2026agitheory}. Using only a ratio functional involving energy (system cost) and entropy (system information gain), the natural-balance principle likewise yields MetaPred as the optimal solution. If this principle holds, it directly asserts, at the level of the mathematical optimization problem, that MetaPred is the most ``economical'' machine-learning paradigm. MetaPred must adopt the general SiFu learning framework in which the model carrier and dynamics are loosely coupled. This further confirms that its model implementations are not constrained by scaling laws---the largest mechanistic bottleneck of current large GPT models: model performance is decoupled from model size, while the model automatically obtains the optimal trade-off between spatial and temporal dynamics from its architectural implementation. Therefore, any MetaPred implementation can decouple input and output lengths from model size. If this theoretical deduction holds, it offers an optimistic prospect for implementing a real AGI system: a runnable AGI system need not be a ``large'' model, nor does it require encyclopedic memory as a prerequisite for AGI capability. If this theoretical deduction holds, it has a sharp implication for the current focus on scaling: general intelligence should be expected from a specific model required by mathematics, rather than from scale; BriLLM-MetaPred is the first instance of this class.

\bibliography{references}

\begin{thebibliography}{14}
\providecommand{\natexlab}[1]{#1}
\providecommand{\url}[1]{\texttt{#1}}
\expandafter\ifx\csname urlstyle\endcsname\relax
  \providecommand{\doi}[1]{doi: #1}\else
  \providecommand{\doi}{doi: \begingroup \urlstyle{rm}\Url}\fi

\bibitem[Arnal \& Giraud(2012)Arnal and Giraud]{arnal2012cortical}
Luc~H Arnal and Anne-Lise Giraud.
\newblock Cortical oscillations and sensory predictions.
\newblock \emph{Trends in Cognitive Sciences}, 16\penalty0 (7):\penalty0
  390--398, 2012.

\bibitem[Bastos et~al.(2012)Bastos, Usrey, Adams, Mangun, Fries, and
  Friston]{bastos2012canonical}
Andre~M Bastos, W~Martin Usrey, Rick~A Adams, George~R Mangun, Pascal Fries,
  and Karl~J Friston.
\newblock Canonical microcircuits for predictive coding.
\newblock \emph{Neuron}, 76\penalty0 (4):\penalty0 695--711, 2012.

\bibitem[Dehghani et~al.(2019)Dehghani, Gouws, Vinyals, Uszkoreit, and
  Kaiser]{dehghani2019universal}
Mostafa Dehghani, Stephan Gouws, Oriol Vinyals, Jakob Uszkoreit, and {\L}ukasz
  Kaiser.
\newblock Universal transformers.
\newblock In \emph{International Conference on Learning Representations}, 2019.
\newblock URL \url{https://openreview.net/forum?id=HyzdRiR9Y7}.

\bibitem[Friston(2010)]{friston2010free}
Karl Friston.
\newblock The free-energy principle: a unified brain theory?
\newblock \emph{Nature Reviews Neuroscience}, 11\penalty0 (2):\penalty0
  127--138, 2010.

\bibitem[Goyal et~al.(2022)Goyal, Didolkar, Lamb, Badola, Ke, Rahaman, Binas,
  Blundell, Mozer, and Bengio]{goyal2022workspace}
Anirudh Goyal, Aniket Didolkar, Alex Lamb, Kartikeya Badola, Nan~Rosemary Ke,
  Nasim Rahaman, Jonathan Binas, Charles Blundell, Michael Mozer, and Yoshua
  Bengio.
\newblock Coordination among neural modules through a shared global workspace.
\newblock In \emph{International Conference on Learning Representations}, 2022.
\newblock URL \url{https://openreview.net/forum?id=9u1uF_B7X1M}.

\bibitem[Gu et~al.(2022)Gu, Goel, and R{\'e}]{gu2022s4}
Albert Gu, Karan Goel, and Christopher R{\'e}.
\newblock Efficiently modeling long sequences with structured state spaces.
\newblock In \emph{International Conference on Learning Representations}, 2022.
\newblock URL \url{https://openreview.net/forum?id=uYLFoz1vlAC}.

\bibitem[Hopfield(1982)]{Hopfield1982}
John~J. Hopfield.
\newblock Neural networks and physical systems with emergent collective
  computational abilities.
\newblock \emph{Proceedings of the National Academy of Sciences}, 79:\penalty0
  2554--2558, 1982.

\bibitem[Lambek(1968)]{Lambek1968}
Joachim Lambek.
\newblock A fixpoint theorem in complete categories.
\newblock \emph{Mathematische Zeitschrift}, 103:\penalty0 151--161, 1968.

\bibitem[Noda et~al.(2017)Noda, Amemiya, Shiramatsu, and
  Takahashi]{noda2017phase}
Takahiro Noda, Tomoki Amemiya, Tomoyo~I Shiramatsu, and Hirokazu Takahashi.
\newblock Stimulus phase locking of cortical oscillations for rhythmic tone
  sequences in rats.
\newblock \emph{Frontiers in Neural Circuits}, 11:\penalty0 2, 2017.

\bibitem[Radford et~al.(2018)Radford, Narasimhan, Salimans, and
  Sutskever]{radford2018improving}
Alec Radford, Karthik Narasimhan, Tim Salimans, and Ilya Sutskever.
\newblock Improving language understanding by generative pre-training.
\newblock \emph{OpenAI blog}, 2018.

\bibitem[Ramsauer et~al.(2021)Ramsauer, Sch{\"a}fl, Lehner, Seidl, Widrich,
  Gruber, Holzleitner, Adler, Kreil, Kopp, et~al.]{ramsauer2020hopfield}
Hubert Ramsauer, Bernhard Sch{\"a}fl, Johannes Lehner, Philipp Seidl, Michael
  Widrich, Lukas Gruber, Markus Holzleitner, Thomas Adler, David Kreil, Michael
  Kopp, et~al.
\newblock Hopfield networks is all you need.
\newblock In \emph{International Conference on Learning Representations
  (ICLR)}, 2021.

\bibitem[Vaswani et~al.(2017)Vaswani, Shazeer, Parmar, Uszkoreit, Jones, Gomez,
  Kaiser, and Polosukhin]{vaswani2017attention}
Ashish Vaswani, Noam Shazeer, Niki Parmar, Jakob Uszkoreit, Llion Jones,
  Aidan~N. Gomez, {\L}ukasz Kaiser, and Illia Polosukhin.
\newblock Attention is all you need.
\newblock In \emph{Advances in Neural Information Processing Systems}, 2017.

\bibitem[Zhao(2026)]{zhao2026agitheory}
Hai Zhao.
\newblock A mathematical theory of {AGI}.
\newblock Under review of ICLR 2027, 2026.

\bibitem[Zhao et~al.(2025)Zhao, Wu, Yang, Zou, and Hong]{zhao2025brillm}
Hai Zhao, Hongqiu Wu, Dongjie Yang, Anni Zou, and Jiale Hong.
\newblock {BriLLM}: Brain-inspired large language model.
\newblock \emph{arXiv preprint arXiv:2503.11299}, 2025.
\newblock Version 8.

\end{thebibliography}
\bibliographystyle{iclr2027_conference}

\appendix

\section{Conditional Convergence and Residual Decomposition Proofs}
\label{sec:proofs}
This section provides the theoretical basis for MetaPred's self-measurement mechanism, including the telescoping decomposition of the residual and the conditional convergence theorem. These results provide conditional mathematical guarantees for endogenous self-supervision; convergence of the model state during training is determined by the specific parameters.

\subsection{Telescoping Decomposition of the Residual}
For brevity, the proofs below omit the condition parameter $\chi$ of single-step evolution, writing $\Phi_q^{\text{step}}(\mathbf{S}) := \Phi_q^{\text{step}}(\mathbf{S}; \chi)$.

Consider a $K$-step simulation $\mathbf{S}_{\text{sim}}^{(K)} = \Phi_q^{(K)}(\mathbf{S}_u)$ starting from executed evolution step $u$. By definition, $\Phi_q^{(K)}$ is a composition of $K$ single-step evolutions. Define the single-step residual as
\[
\Delta_u^{(1)} = \Phi_q^{\text{step}}(\mathbf{S}_u) - \mathbf{S}_u
\]
Then the $K$-step residual $\Delta_{\text{self}}^{(K)}(u) = \mathbf{S}_{\text{sim}}^{(K)} - \mathbf{S}_{\text{pred}}$ (where $\mathbf{S}_{\text{pred}} = \Phi_q^{\text{step}}(\mathbf{S}_u)$) satisfies the following decomposition:

\begin{theorem}[Residual decomposition]
For $K \geq 2$, the $K$-step self-measurement residual can be decomposed into the cumulative sum of the first $K-1$ single-step residuals:
\[
\Delta_{\text{self}}^{(K)}(u) = \sum_{k=1}^{K-1} \Delta_{u+k}^{(1)}
\]
where $\Delta_{u+k}^{(1)}$ is the single-step residual produced at step $k$ along the simulation path.
\end{theorem}

\begin{proof}
By the definition of the composite operator,
\begin{align*}
\mathbf{S}_{\text{sim}}^{(K)} &= \Phi_{q+K-1}^{\text{step}} \circ \dots \circ \Phi_{q+1}^{\text{step}} \circ \Phi_q^{\text{step}}(\mathbf{S}_u) \\
&= \Phi_{q+K-1}^{\text{step}} \circ \dots \circ \Phi_{q+1}^{\text{step}}(\mathbf{S}_{\text{pred}})
\end{align*}
Let $\mathbf{S}_{\text{sim}}^{(k)}$ be the simulated state at step $k$. Then
\begin{align*}
\Delta_{\text{self}}^{(K)}(u) &= \mathbf{S}_{\text{sim}}^{(K)} - \mathbf{S}_{\text{pred}} \\
&= \sum_{k=1}^{K-1} \left( \mathbf{S}_{\text{sim}}^{(k+1)} - \mathbf{S}_{\text{sim}}^{(k)} \right) \\
&= \sum_{k=1}^{K-1} \Delta_{u+k}^{(1)}
\end{align*}
where $\mathbf{S}_{\text{sim}}^{(1)} = \mathbf{S}_{\text{pred}}$.
\end{proof}

This decomposition shows that a $K$-step simulation ($K>1$) captures the cumulative trend of the dynamics over multiple evolution steps rather than a single-step residual signal, providing a theoretical basis for the trend supervision $\mathcal{L}_{\text{dyn}}$.

\subsection{Conditional Convergence Theorem}
MetaPred's self-measurement mechanism draws on the idea of Lambek's fixed-point lemma\citep{Lambek1968}, but guarantees convergence only under specific conditions. We consider the full parameter-group cycle $\mathcal{C}_{q+1} = \Phi_{q+1}^{(Q)}$, which depends on the slow parameters $\Slow$ and runtime conditions $\chi$.

\begin{theorem}[Conditional convergence]
Assume that the full parameter-group cycle $\mathcal{C}_{q+1}$ is Lipschitz continuous on a closed subset $\mathcal{X} \subseteq \mathbb{R}^{N\times d}$, with Lipschitz constant $L < 1$ (i.e., it is a contraction). Then, by the Banach fixed-point theorem, there exists a unique fixed point $\mathbf{S}^* \in \mathcal{X}$ such that $\mathcal{C}_{q+1}(\mathbf{S}^*) = \mathbf{S}^*$, and for any initial state $\mathbf{S}_0 \in \mathcal{X}$, the iteration sequence $\mathbf{S}_{n+1} = \mathcal{C}_{q+1}(\mathbf{S}_n)$ converges linearly to $\mathbf{S}^*$.
\end{theorem}

\begin{proof}
This is a standard application of the Banach fixed-point theorem: a contraction on a complete metric space has a unique fixed point. Here, the metric is induced by the Frobenius norm, and $\mathcal{X}$ is a closed subset (a complete subspace).
\end{proof}

\textbf{Important note}: the theorem above is a conditional result. During training, the slow parameters $\Slow$ are continually updated, and the Lipschitz constant $L$ of the parameter-group cycle $\mathcal{C}_{q+1}$ may exceed 1; convergence of the model to a fixed point during training is conditional. The self-measurement loss $\mathcal{L}_{\text{self}}$ encourages approximately contractive behavior locally, while global convergence is determined by the specific parameters.

\subsection{Geometric Property of the Self-Measurement Correction Term}
The self-measurement correction term $\delta\mathbf{S}_u = \eta_{\text{eff}} \widetilde{\Delta}_{\text{self}}^{(K)}(u)$ has a clear geometric interpretation:

\begin{proposition}[Line-segment constraint on the correction term]
If the simulated residual $\widetilde{\Delta}_{\text{self}}^{(K)}(u) = \Delta_{\text{self}}^{(K)}(u)$ is used, and the effective strength $\eta_{\text{eff}} \in [0,1]$, then the corrected state $\mathbf{S}_{u+1}$ lies on the line segment between the actual successor $\mathbf{S}_{u+1}^{\text{pre}}$ and the $K$-step simulation endpoint $\mathbf{S}_{\text{sim}}^{(K)}$.
\end{proposition}

\begin{proof}
By definition, $\mathbf{S}_{u+1} = \mathbf{S}_{u+1}^{\text{pre}} + \eta_{\text{eff}} (\mathbf{S}_{\text{sim}}^{(K)} - \mathbf{S}_{\text{pred}})$. Since $\mathbf{S}_{u+1}^{\text{pre}} = \mathbf{S}_{\text{pred}}$, we obtain
\[
\mathbf{S}_{u+1} = (1-\eta_{\text{eff}})\mathbf{S}_{\text{pred}} + \eta_{\text{eff}}\mathbf{S}_{\text{sim}}^{(K)}
\]
When $\eta_{\text{eff}} \in [0,1]$, this is a convex combination of the two points and therefore lies on the line segment between them.
\end{proof}

This property restricts state correction to the range of the simulation endpoints, maintaining the stability of the dynamics.

\section{Implementation Details and Update Scope}
\label{sec:implementation}
This section summarizes the engineering implementation details of MetaPred.

\subsection{Network Architecture and Initialization}
\begin{itemize}
    \item \textbf{Node and auxiliary module configurations}:
The state-network size varies across experiments.
The main models reported in this paper use the configurations below, where $N$ is the total number of nodes,
$d$ is the state dimension of each node, $r$ is the bottleneck dimension of the prediction and correction adapters,
and $d_h$ is the state-summary dimension:
\begin{center}
\small
\begin{tabular}{lrrrr}
\toprule
Experimental model & $N$ & $d$ & $r$ & $d_h$ \\
\midrule
Chinese 13M  & 256 & 256 & 8  & 16 \\
English 13M  & 96  & 132 & 8  & 16 \\
Chinese 129M & 256 & 928 & 32 & 32 \\
\bottomrule
\end{tabular}
\end{center}
All these main models use $|I|=16$, $|O|=16$,
$|E|=1$, and $|M|=40$.
The number of working nodes is
$|W|=N-|I|-|O|-|E|-|M|$.
Node masking in structural ablations follows the definitions of the corresponding experiments.
    \item \textbf{Parameter groups}: the total number of parameter groups is $Q=4$, and evolution step $u$ uses parameter-group index $q_u = u \bmod Q$; each parameter group contains independent $\operatorname{GateMix}$ weights, as defined in Eq.~(\ref{eq:gatemix}), and the groups are reused cyclically.
    \item \textbf{Initialization}:
    \begin{itemize}
        \item Input projection matrix $E_{\text{in}}$ and output projection matrix $W_{\text{out}}$: trained from scratch with Xavier uniform initialization.
\item Linear-layer weights: Xavier uniform initialization; biases are initialized to 0.
\item Layer normalization $\operatorname{LN}$: initialized as the identity mapping ($\gamma=1, \beta=0$).
\item Linear transformations in the core gated mixing module: output-layer weights are multiplied by a scaling factor of $0.1$ to prevent initial signals from exploding; newly added auxiliary heads are independently initialized at a scale of $10^{-3}$.
        \item Fast parameters $\Fast_u$: biases are initialized to zero; down-projection weights use Xavier uniform initialization, and up-projection weights use $\mathcal{N}(0, 10^{-3})$; initialization is random and independent of the slow parameters.
        \item Control parameters $\boldsymbol{\psi}$: randomly initialized independently of the other model parameters and trained from scratch.
    \end{itemize}
    \item \textbf{Slot projection vectors}: each input-window position $k$ has two learnable slot projection vectors $s_k$ and $p_k$ ($k=1,\dots,|I|$), corresponding to the slot embedding table and position embedding table in the implementation, respectively. Both are learnable parameters in $\mathbb{R}^d$, initialized from the normal distribution $\mathcal{N}(0, 0.02^2)$ and trained together with the slow parameters $\Slow$.
\end{itemize}

\subsection{Optimizers and Gradient Clipping}
All gradient updates follow the separate update rules defined above, with the following configurations:
\begin{itemize}
    
    \item \textbf{Slow-parameter optimizer}: AdamW with $(\beta_1, \beta_2)=(0.9, 0.999)$ and weight decay $\lambda_w=0.01$; the global gradient-norm limit is $c_{\Slow}=1.0$, implemented by the $\operatorname{ClipNorm}$ operator.
    \item \textbf{Fast-parameter optimizer}: gradient-clipped SGD with learning rate $\alpha_{\Fast}=0.05$ and gradient-norm limit $c_{\Fast}=1.0$; the controller action $a_u^{\text{str}}$ dynamically scales the update strength.
    
    \item \textbf{Control-parameter optimizer}: Adam with learning rate $3\times 10^{-4}$ updates the control parameters $\boldsymbol{\psi}$ using the PPO objective. Slow parameters $\Slow$ use AdamW, and fast parameters $\Fast_u$ use gradient-clipped SGD. The controller architecture and parameter count are given in Section~\ref{sec:controller}.
    \item \textbf{Stop-gradient rules}: gradient boundaries depend on how each quantity is used. Diagnostic simulation-based measurements, the historical observation summary $\mu_u$, and all components of controller input $o_u$ are detached using the $\text{sg}(\cdot)$ operator; when recomputing for supervision, predictions and residuals are recomputed with gradients retained for the slow prediction parameters and slow correction parameters; during training, direct measurements used for state correction are recomputed after the fast-parameter update, and when outer-loop gradients are enabled, the corresponding gradients are retained through the transition to the next state; reference signals from the $K=2$ and $K=Q+1$ simulations are detached. Simulation-based measurements retain gradients in training mode (\textsc{Train}) and train the shared dynamics through the self-measurement loss; in evaluation modes, gradients are not computed through them.
\end{itemize}

\subsection{Controller Network Implementation Details}
\label{sec:controller}

The controller network first processes the causal state summary $o_u$ with learnable layer normalization $\operatorname{LN}$, followed by two hidden layers, each of dimension 64 with Tanh activation:
\[
z_u=\tanh\!\left(W_2\tanh\!\left(W_1\operatorname{LN}(o_u)+b_1\right)+b_2\right),
\]
where $W_1\in\mathbb{R}^{64\times 66}$ and $W_2\in\mathbb{R}^{64\times 64}$. The shared representation $z_u$ feeds six action heads and one scalar value head. The six action heads produce discrete distributions over the measurement tier, measurement object, fast-parameter update switch, fast-parameter update strength, state-correction strength, and fast-parameter update sampling ratio, respectively. The numbers of categories are $4,9,2,3,3,3$, respectively, for a total of 24 category scores; the value head outputs the scalar estimate $V_{\boldsymbol{\psi}}(o_u)$.

The causal state summary $o_u$ has 66 dimensions and concatenates the following nine components:
\begin{itemize}
    \item Control state summary, 8 dimensions ($s_c(\mathbf{S}_u)$): the global mean, RMS, RMS after centering along the node axis, the RMS values of input/output/internal/self-coordination nodes, and the proportion of finite values;
    \item Recent runtime signals, 8 dimensions: prediction error, intrinsic energy, the RMS values of velocity and acceleration, the norms of the fast-parameter update and state correction, the previous reward, and the parameter version;
    \item Compatible measurement history, 8 dimensions ($\mathcal{Q}_u$): difference statistics of existing compatible measurements and the corresponding validity information;
    \item Measurement age, 9 dimensions: the normalized number of steps since the last measurement for each of the nine classes of measurement objects;
    \item Missing-record flags, 9 dimensions ($\zeta_u$): whether each of the nine classes of measurement objects has no record yet;
    \item Previous action, 6 dimensions: the indices of the six previous action components;
    \item Stage flags, 3 dimensions: a one-hot encoding of the invocation stage (encoding, autonomous, generation);
    \item Remaining budgets, 2 dimensions: the remaining fractions of the measurement budget and fast-parameter update budget;
    \item Runtime diagnostics, 13 dimensions: phase, sample-pool and prediction-queue occupancy, first- and second-order energy differences and their validity, self-measurement residual and loss, calibration status, and related diagnostics.
\end{itemize}
The total is $8+8+8+9+9+6+3+2+13=66$. When constructing the observation, selected continuous statistics first undergo signed logarithmic compression $f(x)=\operatorname{sign}(x)\log(1+|x|)$, followed by the learnable layer normalization within the controller network.

The nine classes of measurement objects are self-consistency, input nodes, output nodes, internal nodes, the self-coordination node, direct-measurement parameters, fast parameters, the state network, and the policy.
These classes correspond, in order, to the nine positions in the causal state summary components $\tau_u$ and $\zeta_u$ and to the nine candidate categories of $a_u^{\text{focus}}$.
$a_u^{\text{focus}}$ specifies the measurement object to focus on, while the measurement tier determines the actual measurement scope: direct measurement covers self-consistency and the focus object; short-horizon simulation additionally covers input, output, and internal nodes, as well as fast parameters; long-horizon simulation covers all nine classes, in which case $a_u^{\text{focus}}$ does not affect the set of objects.

The categories and values of the controller action components are as follows:
\begin{itemize}
    \item $a_u^{\text{tier}}$: 4 categories---skip measurement, direct measurement, short-horizon simulation, and long-horizon simulation;
    \item $a_u^{\text{focus}}$: one of the 9 registered measurement objects;
    \item $a_u^{\text{upd}}$: 2 categories---no update and request an update;
    \item $a_u^{\text{str}}$: 3 tiers, $0.25,0.5,1.0$;
    \item $a_u^{\text{corr}}$: 3 tiers, $0,0.25,0.5$;
    \item $a_u^{\text{ratio}}$: 3 tiers, $25\%,50\%,100\%$.
\end{itemize}
When the update switch is off, the sampling-ratio action head index is set to zero; budgets and available records also constrain execution, and requesting an update and executing it are two different steps.

The control parameters $\boldsymbol{\psi}$ comprise the layer normalization, two hidden layers, six action heads, and one value head, totaling approximately $1.02\times 10^4$ parameters. An independent Adam optimizer trains these parameters using the PPO objective; slow parameters $\Slow$ use AdamW, and fast parameters $\Fast_u$ use gradient-clipped SGD.

\subsection{Controller and Self-Measurement Scheduling}
\begin{itemize}
    \item \textbf{Measurement tiers}: $a_u^{\text{tier}} \in \{\textsc{skip}, \textsc{direct}, \textsc{short}, \textsc{long}\}$ corresponds to four tiers: skip measurement, direct measurement, short-horizon simulation, and long-horizon simulation. The short- and long-horizon tiers subdivide simulation-based measurement by depth. Skipping measurement returns a zero signal; direct measurement estimates the residual using the direct-measurement parameters without running a simulation; simulation-based measurement simulates the state, with depth determined by $K$, where $K=2$ denotes short-horizon simulation and $K=Q+1$ denotes long-horizon simulation. The current implementation uses $Q=4$, giving $K=5$.
    \item \textbf{Budget constraints}: the initial measurement budget $B_0^m$ and update budget $B_0^w$ depend on the experiment size. The current implementation uses $512/32$ or $4096/256$. The measurement tiers incur different costs, $(0,1,4,12)$, and the cost of fast-parameter updates also depends on the number of sampled observations; the corresponding operation is automatically disabled when its budget is exhausted.
    \item \textbf{Warm-up and coverage}: warm-up is measured in optimizer steps, with simulation-based measurement preferred during the first several steps; if the direct-measurement parameter calibration error exceeds the threshold (0.5), corrections driven by direct measurement are disabled, and simulation-based measurement may be forced; two long-horizon simulations must be separated by a minimum interval of 16.
\end{itemize}

\subsection{Loss Weights and Hyperparameters}
The multitask weights in Eq.~(\ref{eq:outer-loss}) are configured as follows:
\begin{itemize}
    \item Overall endogenous-loss weight: $\lambda_{\text{int}}=0.01$.
    \item Component loss weights: $(\lambda_p, \lambda_f, \lambda_d, \lambda_c, \lambda_s)=(1.0, 0.1, 0.5, 1.0, 0.1)$.
    \item Dynamical trend loss weights: $(w_s, w_m, w_e)=(1.0, 0.01, 0.1)$ for the state, measurement, and energy trends, respectively.
    \item PPO hyperparameters: GAE parameters $(\lambda_{\text{GAE}}, \gamma)=(0.95, 0.99)$; clipping range $\bar{\rho}_u \in [0.8, 1.2]$; value-function coefficient $c_V=0.5$; entropy coefficient $c_H=0.01$.
    \item Slow-update preview tolerances: $\epsilon_{\text{abs}}=10^{-4}$ and $\epsilon_{\text{rel}}=10^{-3}$.
\end{itemize}

\section{Reproducibility Conventions}
\label{sec:repro}

This section specifies key implementation details, including random-number generation, checkpoint saving, and evaluation conventions, to ensure the reproducibility of the BriLLM-MetaPred experiments.

\subsection{Full Configurations and Reproduction Records}
\label{sec:exp-details}

The preceding sections present the main comparisons and their interpretation. This section collects the training configurations and full-precision values for verification.

\subsubsection{Model Implementation, Data, and Training Configurations}
\label{sec:exp-config}

The experimental implementation is based on BriLLM-MetaPred's state network, controller network, and training algorithm. Experiments differ in training regime, parameter scale, training length, and evaluation mode.

\begin{table}[htp]
\caption{Main experimental configurations. Parameter counts are based on registered parameters.}
\label{tab:exp-overview}
\centering\small
\setlength{\tabcolsep}{3pt}
\renewcommand{\arraystretch}{1.15}
\begin{tabular}{@{}>{\raggedright\arraybackslash}p{\dimexpr0.32\linewidth-6pt\relax}>{\raggedright\arraybackslash}p{\dimexpr0.14\linewidth-6pt\relax}>{\raggedright\arraybackslash}p{\dimexpr0.16\linewidth-6pt\relax}>{\raggedright\arraybackslash}p{\dimexpr0.17\linewidth-6pt\relax}>{\raggedright\arraybackslash}p{\dimexpr0.21\linewidth-6pt\relax}@{}}
\toprule
\textbf{Experiment} & \textbf{Parameters} & \textbf{Node count/width} & \textbf{Vocabulary} & \textbf{Actual training tokens} \\
\midrule
BookCorpus continuous-execution training regime & 13.2226M & 96/132 & 40,478 BPE & 25,014,272  \\
THUCNews sparse-auxiliary training regime 13M & 13.0945M & 256/256 & 6,512 & 25,000,960  \\
THUCNews sparse-auxiliary training regime 129M & 129.3854M & 256/928 & 6,512 & 470,404,384  \\
\bottomrule
\end{tabular}
\end{table}

\paragraph{BookCorpus data splits and preprocessing.}
After whole-book SHA256 exact deduplication, 17,785 of BookCorpusOpen's 17,868 original documents remain. They are split by document hash in a 98:1:1 ratio, yielding 17,332/167/176 books in the training/validation/test sets after preprocessing. The vocabulary contains the 40,478 entries of the official GPT-1 BPE vocabulary. Preprocessing uses spaCy 3.8.7 blank-en and ftfy 6.3.1.

Blocks remain within a single book, with both length and stride set to 512; short tails are discarded. The 496 next-token targets at zero-based positions 16--511 are evaluated. The training set contains 2,903,718 blocks and 1,440,244,128 training tokens; the full validation set contains 25,540 blocks and 12,667,840 tokens; test-set metadata records 29,844 blocks and 14,802,624 tokens. Each step's global batch size of 64 corresponds to 31,744 training tokens, and 788 steps give an actual training-token count of 25,014,272.

For the 13M comparison models, BriLLM-MetaPred's state network has a node count/width of 96/132, with 16 nodes each in the input and output regions, 40 modal nodes, and FFN width 528. The GPT baseline has 12 layers, width 192, 3 attention heads, FFN width 768, and dropout 0.1.

The shared optimizer is AdamW, with peak learning rate $2.5\times10^{-4}$, 40 steps of linear warm-up followed by cosine decay, $\beta=(0.9,0.999)$, $\epsilon=10^{-8}$, weight decay 0.01, and global gradient clipping 1. Training uses FP32 with TF32 enabled; formal evaluation disables TF32. BriLLM-MetaPred uses SDAA's $64\times1$ topology and flattened CPU FP32 Gloo mean synchronization; GPT uses A100's $2\times32$ CUDA DDP.

The number of autonomous evolution steps in BriLLM-MetaPred is 5, and each training token uses 1 generation evolution step. The fast-parameter rank is 8, the local learning rate is 0.05, and the clipping cap is 1. The history capacity is 64, the candidate-record pool size is 8, and the pending-prediction capacity is 8. Short-horizon/long-horizon simulations use 2/5 steps, the direct-measurement calibration interval is 8, the measurement/update budgets are 4096/256, and the coverage interval is 32. Correction strengths are $\{0,0.25,0.5\}$, with an absolute RMS cap of 0.05 and a relative natural-evolution cap of 0.25.

The internal objective coefficients are $\lambda_{\mathrm{pred}}=1$, $\lambda_{\mathrm{fp}}=0.1$, and $\lambda_{\mathrm{dyn}}=0.5$, with outer intrinsic loss coefficients of 0.01/0.1. The state, measurement, and energy trend weights are 1/0.01/0.1; the local state trend coefficient is 0.001 and the direct-measurement calibration coefficient is 1. The independent controller uses Adam and executes PPO every 64 steps, with 2 epochs, mini-batch size 32, clipping 0.2, discount 0.99, GAE 0.95, value coefficient 0.5, and entropy coefficient 0.01. The training reward additionally includes task cross-entropy feedback weighted by 0.01. Formal evaluation freezes the controller and stops PPO data collection.

\paragraph{THUCNews data splits.}
THUCNews-v3 is split in a 90:5:5 ratio by exact and conservative near-duplicate article groups, with 728,267/40,382/40,533 articles in the training/validation/test sets. The vocabulary consists of tokens occurring at least 5 times in the training set and contains 6,512 entries. Full hyperparameters and optimizer configurations are provided in Appendix~\ref{sec:exp-details}.

\paragraph{Execution details of the sparse-auxiliary training regime.}
The 13M model in the sparse-auxiliary training regime uses a state network with 256 nodes and node width 256, and is trained on 25,000,960 tokens.
The 129M model's state network has 256 nodes and node width 928.

\subsection{Random-Number Generation and Management}
MetaPred involves two types of randomness: the inherent randomness of parameter initialization, and random operations during training/inference, such as policy sampling and dropout. The following mechanisms are used to ensure reproducibility:
\begin{itemize}
    \item \textbf{Global RNG}: used for parameter initialization and controlled by the main seed (e.g., 42).
    \item \textbf{Runtime RNG}: each sequence runtime $z_u$ contains an independent RNG state $\xi_u$. This state is updated as the evolution step advances, ensuring that
    \begin{enumerate}
        \item the random streams for simulated evolution (self-measurement) and actual state evolution are fully isolated;
        \item different operating modes of the same sequence (\textsc{Train}/\textsc{Frozen}) produce identical random decisions, such as dropout masks, at the same evolution step, provided the RNG state is correctly saved and restored.
    \end{enumerate}
    \item \textbf{Policy sampling}: the controller samples actions using the runtime RNG, ensuring causal consistency, i.e., sampling depends only on current and past information.
\end{itemize}

\subsection{Checkpoint Format and Contents}
A training checkpoint must save the following states to support resuming training from any evolution step or performing causally consistent evaluation:
\begin{itemize}
    \item Slow parameters $\Slow$ (including all learnable weights and biases).
    \item Control parameters $\boldsymbol{\psi}$.
    \item Optimizer states (AdamW's first- and second-moment estimates).
    \item The current evolution-step counter $u$, token-step counter $t$, and parameter-group index $q$.
    \item The full state $z_u$ of the most recent sequence runtime (including $\mathbf{S}_u, \Fast_u, \mathcal{H}_u, \mathcal{B}_u, \xi_u$), for resuming online evaluation.
\end{itemize}

Checkpoint filenames include the global step count and a state label. Candidate updates rejected by the slow-update preview mechanism are saved separately.

\subsection{Evaluation Conventions}
MetaPred's three operating modes (\textsc{Train}, \textsc{Frozen}, \textsc{Adaptive}) serve different evaluation purposes, with the following conventions:
\begin{itemize}
    \item \textbf{\textsc{Train}-mode diagnostics}: the task loss $\mathcal{L}_{\text{task}}$ and self-measurement loss $\mathcal{L}_{\text{self}}$ are computed on the validation set to monitor training convergence. These diagnostics use the full \textsc{Train} configuration, including fast-parameter updates and state correction, while keeping the slow parameters fixed; they provide a reading of the training-process state.
    
    \item \textbf{\textsc{Frozen}-mode evaluation}: all parameters ($\Slow, \Fast_u, \boldsymbol{\psi}$) are frozen, fast-parameter updates and state correction are disabled, and dynamical evolution, self-measurement, and prediction readout are executed. External task loss can still be computed in this mode, but only to evaluate model performance. The self-measurement residual directly measures the model's internal consistency at its current state. Controller actions are still computed and recorded but have no actual effect.
    
    \item \textbf{\textsc{Adaptive}-mode evaluation}: this mode simulates actual deployment. The slow parameters $\Slow$ and control parameters $\boldsymbol{\psi}$ are frozen, but fast-parameter updates and state correction are allowed. During evaluation, the sequence runtime state $z_u$ is reset at sequence boundaries. This mode measures the model's adaptive capability in online-learning scenarios.
\end{itemize}

\subsection{Hardware and Numerical Precision}
\begin{itemize}
    \item \textbf{Hardware environment}: experiments are performed on NVIDIA A100, A800, and SDAA, using CUDA 11.8 or later.
\item \textbf{Numerical precision}: all current formal experiments use \texttt{float32} (FP32) precision.
    \item \textbf{Deterministic algorithms}: experiments with particularly strict reproducibility requirements enable CuDNN deterministic mode (\texttt{CUDNN\_DETERMINISTIC=1}), at the cost of some training speed.
\end{itemize}

\subsection{Code and Data Availability}
The model is implemented in PyTorch. Training uses the public datasets BookCorpusOpen, text8, and THUCNews-v3 under their respective data-use agreements. Preprocessing scripts, model definitions, and training/evaluation code will be released with the paper to ensure full reproducibility.

\end{CJK}
\end{document}